%% file: revision.tex
\documentclass[12pt]{article}
\usepackage{amsmath}
\usepackage{graphicx,psfrag,epsf}
\usepackage{enumerate}
\usepackage{natbib}
\usepackage{xcolor}

\newcommand{\diag}{\textrm{diag}}
\newcommand{\tr}{\textrm{tr}}

\newcommand{\spa}{\textrm{sp}}
\newcommand{\sign}{\textrm{sign}}

\usepackage{amsthm,mathtools,amsfonts,amssymb}
\usepackage{booktabs}
\usepackage{subcaption}
\usepackage{amsbsy}
\usepackage{parskip}
\usepackage{sectsty}
\usepackage{setspace}
\usepackage{url} 
\usepackage[utf8]{inputenc}
\usepackage[english]{babel}
\usepackage{longtable}
\usepackage{authblk}
\usepackage{animate}
\usepackage{etoolbox}
\usepackage{titlesec}
\usepackage{tipa}
\usepackage{float}
\usepackage{nameref}
\usepackage{multirow,setspace}
\usepackage{fancyhdr}
\usepackage[ruled,vlined,linesnumbered,resetcount,algosection]{algorithm2e}
\usepackage{array}
\usepackage{algpseudocode}
\usepackage{comment}
\usepackage{adjustbox}
\usepackage{bookmark}
\RequirePackage[OT1]{fontenc}
\RequirePackage{bm,hyperref}
\usepackage{placeins}
\renewcommand{\thefigure}{\arabic{section}.\arabic{figure}}
\counterwithin{figure}{section}
\numberwithin{equation}{section}

\makeatletter
\patchcmd{\@tocline}{\hfill}{%
  \leaders\hbox{$\m@th
    \mkern \@dotsep mu\hbox{\normalfont{.}}\mkern \@dotsep
  mu$}\hfill}{}{}
\makeatother
\makeatletter
\patchcmd{\l@section}{\hfil}{%
  \leaders\hbox{$\m@th
    \mkern \@dotsep mu\hbox{\normalfont{.}}\mkern \@dotsep
  mu$}\hfill}{}{}
\makeatother
\newcommand{\pkg}[1]{{\normalfont\fontseries{b}\selectfont #1}}

\newcommand\BibTeX{{\rmfamily B\kern-.05em \textsc{i\kern-.025em b}\kern-.08em
T\kern-.1667em\lower.7ex\hbox{E}\kern-.125emX}}

\newtheorem{theorem}{Theorem}[section]
\newtheorem{corollary}[theorem]{Corollary}
\newtheorem{lemma}[theorem]{Lemma}
\newtheorem{proposition}[theorem]{Proposition}

\titleformat{\paragraph}{\normalfont\normalsize\bfseries}{\theparagraph}{1em}{}
\titlespacing*{\paragraph}{0pt}{3.25ex plus 1ex minus .2ex}{1.5ex plus .2ex}
\DeclareMathAlphabet\mathbfcal{OMS}{cmsy}{b}{n}

\newcommand{\blind}{0}

\graphicspath{{images/}{./}}

\begin{document}

\def\spacingset#1{\renewcommand{\baselinestretch}%
{#1}\small\normalsize} \spacingset{1}


\if0\blind
{
  \title{\bf A Functional SVD Framework for Regularized Multivariate Functional PCA with Dual Penalization}
  \author{Yue Zhao$^1$, Hossein Haghbin$^2$, Rebecca Sanders$^3$, and Mehdi Maadooliat$^3$\thanks{Corresponding Author: mehdi.maadooliat@mu.edu} \\
    $^1$Division of Biostatistics and Health Data Science,\\ University of Minnesota, USA\\ \vspace{.1in}
    $^2$Faculty of Intelligent Systems Engineering and Data Science,\\ Persian Gulf University, Iran \\ \vspace{.1in}
    $^3$Department of Mathematical and Statistical Sciences,\\ Marquette University, USA\\ 
 }
    \date{}
  \maketitle
} \fi

\if1\blind
{
  \bigskip
  \bigskip
  \bigskip
  \begin{center}
    {\LARGE\bf Title}
\end{center}
  \medskip
} \fi
\vspace{-.1in}
\begin{abstract}
This paper introduces a novel framework for Regularized Multivariate Functional Principal Component Analysis (ReMFPCA) via Functional Singular Value Decomposition (SVD). The proposed method extends existing MFPCA approaches by incorporating a generalized functional SVD within a Hilbert space framework, enabling simultaneous regularization of both functional principal components (PCs) and their associated PC scores. A key innovation of this framework is the inclusion of a sparsity penalty on the PC scores, which enhances interpretability by filtering out irrelevant subject-specific variations. This dual-penalization strategy represents a significant advancement beyond existing covariance-based eigen decomposition methods, which penalize only the functional PCs. Two power algorithm implementations, sequential and joint, are proposed, together with a cross-validation approach based on iterative regression for optimal smoothing parameter selection. Comprehensive simulation studies and real data applications demonstrate that the proposed framework substantially improves the extraction of informative and interpretable components, offering methodological and practical benefits for analyzing multivariate functional data across diverse domains.
\end{abstract}

\noindent%

\spacingset{1.45}
\section{Introduction}
\label{sec:intro}
In recent years, interest in functional data analysis (FDA) has grown significantly due to its exceptional ability to manage complex data structures, such as longitudinal, time series, and biomedical imaging data. Functional Principal Component Analysis (FPCA), as a key technique in FDA, is widely used to identify primary patterns within functional observations, encapsulated in functional principal components (PCs). A comprehensive study of FPCA is available in \cite{ramsay2005}. An extensive literature review of FPCA has been conducted by \cite{shang2014survey}. 

Despite the wide applicability and effectiveness of FPCA, it is not without limitations. One significant challenge is the interpretation of the estimated functional PCs, particularly when the true functional PCs are characterized by issues such as roughness and sparsity. To address roughness, various extensions of FPCA have been proposed. \cite{BesseFPCA}, \cite{RAMSAYFPCA} and \cite{KNEIPFPCA} explored methods for smoothing functional data before applying FPCA. \cite{Joan}, \cite{yao2003}, \cite{Di} and \cite{Goldsmith} applied kernel smoothing or penalized spline techniques to the covariance function. \cite{rice1991estimating}, \cite{silverman1996smoothed}, \cite{huang2008functional} and \cite{lakraj2017some} introduced roughness penalties and proposed regularized Functional Principal Component Analysis (ReFPCA) approaches. To address sparsity, \cite{lei}, \cite{interpretableFPCA} and \cite{spfca} introduced sparsity penalties to balance interpretability and variance explanation. Among these approaches, \cite{huang2008functional} applied singular value decomposition (SVD) to the dataset, while most others rely on eigen decomposition of covariance functions.  

Multivariate FPCA (MFPCA) extends FPCA to multivariate functional data and permits the analysis of interactions among various functional variables. \cite{ramsay2005} proposed an MFPCA approach for bivariate functional data, assuming the variables are measured in the same units and exhibit similar variability. \cite{chiou2014multivariate} overcame these limitations by introducing a normalized MFPCA method. \cite{happ2018multivariate} developed a generalized MFPCA approach that enables the analysis of multivariate functional data across various domains and different dimensions. \cite{jacques2014model} utilized MFPCA for clustering applications. 

Similar to univariate FPCA, existing MFPCA approaches encounter issues with roughness in functions. To address this challenge, various extensions of MFPCA have been proposed. \cite{happ2018multivariate} applied the popular Principal Components Analysis through Conditional Expectation (PACE) algorithm by \cite{PACE}, incorporating penalized splines to smooth the covariance function in a multivariate setting. \cite{kayano2009functional} extended the regularization concept from \cite{silverman1996smoothed} and utilized Gaussian basis expansion to manage the roughness of multivariate functional data under the assumption of a `same domain' condition. \cite{ReMFPCA} explored the mathematical foundation of Regularized MFPCA within a Hilbert space framework, independent of basis choice, and introduced a novel approach for multivariate functional data with `different domains.' Notably, all these approaches rely on the eigen decomposition of covariance functions.

The ReMFPCA introduced by \cite{ReMFPCA} incorporates roughness penalties on the functional PCs, but it does not address an important limitation: the absence of regularization on the PC scores. In practice, this omission results in nonzero scores even for subjects unrelated to a given functional pattern, thereby reducing interpretability and introducing noise. This issue parallels challenges seen in sparse PCA \citep{spca_zou,spca_shen}, where nonzero PC loadings complicate the interpretation of derived PCs. Similarly, nonzero PC scores in FPCA and MFPCA can introduce uncorrelated information as noise, making it more difficult to extract optimal PCs with clearer and more meaningful interpretations. The present paper makes three major contributions beyond existing ReMFPCA frameworks:
\begin{enumerate}
    \item {Generalized Functional SVD for Multivariate Functional Data:} We develop a data-operator-based functional SVD formulation for regularized MFPCA in a product Hilbert space across multiple domains, moving beyond covariance-based eigen decomposition.  
    \item {Dual Regularization on PCs and PC Scores:} In contrast to existing methods including \cite{ReMFPCA} that penalize only the smoothness of functional PCs, the proposed framework introduces an additional sparsity penalty on the subject-specific PC scores. This sparsity constraint is designed to enhance interpretability by allowing irrelevant subject-specific contributions to be shrunk toward zero, thereby yielding components that more clearly represent shared functional patterns rather than noisy individual effects.  
    \item {Efficient Cross-Validation via Regression Connection:} The connection between generalized functional SVD and a regression framework enables the development of a highly efficient cross-validation procedure. This procedure allows distinct smoothing parameters to be selected for different functional PCs, improving both flexibility and computational efficiency.  
\end{enumerate}

In addition, efficient algorithms (sequential and joint power methods) are developed to estimate the components. Simulation studies and real-data applications demonstrate that the proposed framework substantially improves interpretability and performance over prior MFPCA methods.

\cite{Yang2011} is closely connected to the present work through its operator-based use of functional singular value decomposition; however, the two approaches differ in both the operators under study and their inferential goals. \cite{Yang2011} develop a singular decomposition of the cross-covariance operator for bivariate functional data in order to study dependence between two functional processes. In contrast, the proposed ReMFPCA framework is formulated through the singular decomposition of the sample data operator $X:\mathbb{R}^n \to H$ for multivariate functional observations. Accordingly, their framework is tailored to dependence analysis, whereas ours yields a direct SVD-based formulation of multivariate FPCA and its regularized extensions.

The paper is organized as follows: Sections \ref{Methodology} and \ref{Regularized MFPCA via Functional SVD} discuss the mathematical foundations of our approach. Section \ref{Power_Algorithm} discusses the power algorithm for implementation and tuning parameter selection strategy. Section \ref{sect::variability} provides a framework to quantify the variance explained by non-orthogonal PCs extracted from the power algorithm. To demonstrate the effectiveness of the proposed method, Section \ref{Simulation} and \ref{Real data analysis} presents simulation results and an analysis of a real data example. The paper concludes with a discussion and an outlook in Section \ref{Discussion}.

\section{Methodology of Functional SVD}
\label{Methodology}
This section begins with the mathematical foundations utilized throughout this section. The Hilbert space setup and related operator-theoretic concepts introduced below are standard in functional data analysis; see \cite{hsing2015theoretical} for foundational definitions and background. We consider $\mathcal{T}_j,\ j=1,\ldots, p,$ as a compact subset of $\mathbb{R}$, where $p$ denotes the number of variables. Let's define $H_j:=L^2(\mathcal{T}_j)$, consisting of measurable functions $x(\cdot)$ on $\mathcal{T}_j$ with $\int_{\mathcal{T}_j} |x(t)|^2 dt < \infty$. Note that $H_j$ is a Hilbert space equipped with the inner product $\langle x , y \rangle_{H_j}=\int_{\mathcal{T}_j} x(t)y(t) dt$ and the norm $\Vert x \Vert_{H_j}=\sqrt{\langle x , x \rangle_{H_j}}$. The tensor (outer) product $x\otimes y$ corresponds to the operator $x\otimes y: {H}_i\rightarrow {H}_j$, defined as ($x\otimes y)h:=\langle x, h \rangle_{H_i} y$, where $x,h\in {H}_i$ and $y\in {H}_j$. Denote the cartesian product space $\mathbb{H}=H_1\times H_2\times \cdots \times H_p$ equipped with the inner product $\langle {\pmb x,\pmb y} \rangle_\mathbb{H}=\sum_{j=1}^p\langle  x_j,y_j \rangle_{H_j}$. For positive integers $p\ \text{and}\ n$, $\mathbb{F}^{p \times n}$ is the space of all linear operator $\bm{\mathcal{Z}}: \mathbb{R}^n \rightarrow \mathbb{H}$. Any operator $\bm{\mathcal{Z}}$ can be 
specified by $[\pmb{z}_{i}]_{i=1}^{n}$ where
	\begin{equation*}\label{eq: z operator}
		\bm{\mathcal{Z}}\pmb{a}=\sum_{i=1}^n a_i\pmb{z}_{i}
		,
		\ \pmb{z}_{i}\in \mathbb{H},\ \text{and}\ \pmb{a}=(a_1,\ldots, a_n)^{\top}\in\mathbb{R}^n.
	\end{equation*}
	
	For two given operators $\bm{\mathcal{Z}}_1=[\pmb{z}_{i}^{(1)}]_{i=1}^{n}$ and $\bm{\mathcal{Z}}_2=[\pmb{z}_{i}^{(2)}]_{i=1}^{n}$ in $\mathbb{F}^{p\times n}$, define 
	\begin{equation*}
		\langle\bm{\mathcal{Z}}_1,\bm{\mathcal{Z}}_2\rangle_\mathbb{F}:=\sum_{i=1}^n \langle{\pmb{z}}_{i}^{(1)},{\pmb{z}}_{i}^{(2)}\rangle_\mathbb{H}.
	\end{equation*}
	It follows immediately that $\langle \cdot , \cdot \rangle_\mathbb{F}$, defines an inner product. We will call it the Frobenius inner product of two operators in $\mathbb{F}^{p\times n}$. The associated Frobenius norm is $\Vert \bm{\mathcal{Z}} \Vert_\mathbb{F}=\sqrt{\langle\bm{\mathcal{Z}},\bm{\mathcal{Z}}\rangle_\mathbb{F}}$.

\subsection{Data Operator}
Consider the centered sample \(\pmb{x}_i = (x_{i,1}, \ldots, x_{i,p})\) in \(\mathbb{H}\), defined on the domain \(\mathbfcal{T} := \mathcal{T}_1 \times \cdots \times \mathcal{T}_p\). For the collection of samples \(\pmb{x}_1, \ldots, \pmb{x}_n\), we define the operator \(\mathbfcal{X} := [\pmb{x}_i]_{i=1}^{n}\) in \(\mathbb{F}^{p \times n}\). This operator is referred to as the \textit{data operator}.
	
	\begin{proposition}\label{prop:trajO}
The data operator $\mathbfcal{X}$ is a bounded linear operator. If we define $\mathbfcal{X}^*:\mathbb{H} \rightarrow \mathbb{R}^n$, given by
	\begin{equation*}\label{eq:adj.traj}
			\mathbfcal{X}^*{\pmb z}=
			\begin{pmatrix} \langle {\pmb x}_1,{\pmb z}\rangle_{\mathbb{H}}, \langle {\pmb x}_2,{\pmb z}\rangle_{\mathbb{H}}, \ldots, \langle {\pmb x}_n,{\pmb z}\rangle_{\mathbb{H}}\end{pmatrix}^\top,
			\ {\pmb z}\in\mathbb{H},
		\end{equation*}
		then $\mathbfcal{X}^*$ is an adjoint operator for $\mathbfcal{X}$.
	\end{proposition}
	
\begin{proposition} \label{prop 2.2}
For the data operator $\mathbfcal{X}$, the following relationship holds:
\begin{equation*}
    \mathbfcal{X} \circ \mathbfcal{X}^* = (n-1) \widehat{\mathbfcal{C}},
\end{equation*}
where $\circ$ denotes the composition of operators, and $\widehat{\mathbfcal{C}}$ represents the sample estimator of the covariance operator.
\end{proposition}

Clearly, $R(\mathbfcal{X})=\spa\left\lbrace {\pmb x}_i\right\rbrace_{i=1}^{n},$ where $R(\mathbfcal{X})$ denotes the range of $\mathbfcal{X}.$ Therefore, $\dim(R(\mathbfcal{X})) := L \leq n$ that result in $\mathbfcal{X}$ (and all operators $\mathbfcal{Z}\in\mathbb{F}^{p\times n}$) being an operator of rank $L$. Note that every finite-rank operator is also compact, and so bounded. 
	
\subsection{Functional SVD Framework}
In this section, we investigate the ReMFPCA problem by applying the SVD technique within a Hilbert space framework to the data operator, a method we refer to as functional SVD. To effectively pursue this approach, it is crucial to first establish the mathematical foundations of the functional SVD technique.
	\begin{theorem}[]\label{th-fsvd}
		Consider the data operator $\mathbfcal{X}$. There exist linearly independent elements $\pmb{\psi}_1, \ldots,  \pmb{\psi}_L$ from $\mathbb{H}$ and $\pmb{u}_1, \ldots, \pmb{u}_L$ from $\mathbb{R}^n$  that are orthonormal and
		\begin{equation}\label{eq-svdl}
			\mathbfcal{X}=\sum_{\ell=1}^L \sqrt{\lambda_\ell} \ \pmb{u}_\ell\otimes \pmb{\psi}_\ell,
		\end{equation}
		where $\lambda_\ell$'s are  non-ascending positive scalars. Then
		\begin{equation*}\label{eq-svdr}
			\mathbfcal{X}^*=\sum_{\ell=1}^L \sqrt{\lambda_\ell} \ \pmb{\psi}_\ell\otimes\pmb{u}_\ell. 
		\end{equation*}
	\end{theorem}
One may consider Theorem \ref{th-fsvd} as the functional SVD of $\mathbfcal{X}$. We shall refer to $\sqrt{\lambda_\ell}$ as singular value, $\pmb{\psi}_\ell$ as the left singular function, and $\pmb{u}_\ell$ as the right singular vector of the data operator. Therefore, the collection $\left(\sqrt{\lambda_\ell}, \pmb{\psi}_\ell, \pmb{u}_\ell\right)$ can be referred to as the $\ell^{th}$ singular triplet of $\mathbfcal{X}$.
	\begin{proposition}\label{prop-svd}
    In Theorem \ref{th-fsvd}, the set $\{\pmb{\psi}_\ell\}_{\ell=1}^L$ forms a basis for $R(\mathbfcal{X})$, and each $\pmb{u}_i$ can be expressed as:
		\begin{equation*}\label{eq-r.sing.v}
			\pmb{u}_\ell=\dfrac{\mathbfcal{X}^*(\pmb{\psi}_\ell)}{\sqrt{\lambda_\ell}}=\begin{pmatrix} \dfrac{\langle\pmb{\psi}_\ell, {\pmb x}_1\rangle_{\mathbb{H}}}{\sqrt{\lambda_\ell}} , \ldots,\dfrac{ \langle\pmb{\psi}_\ell, {\pmb x}_n\rangle_{\mathbb{H}}}{\sqrt{\lambda_\ell}} \end{pmatrix}^\top,\qquad \ell=1,\ldots,L.
		\end{equation*}
	\end{proposition}
	
\begin{theorem}\label{col1}
Let the set $\{\pmb{\psi}_\ell\}_{\ell=1}^L$ represent the left singular functions of the data operator $\mathbfcal{X}$ in \eqref{eq-svdl}. Then the following holds:
\begin{enumerate}
    \item[i)] The function $\pmb{\psi}_\ell$ is an eigenfunction of the estimated covariance operator ${\widehat{\mathbfcal{C}}}$ with the corresponding eigenvalue $\dfrac{\lambda_\ell}{n-1}$.
    
    \item[ii)] $ \max_{\pmb{\varphi}}{\langle {\widehat{\mathbfcal{C}}}\pmb{\varphi}, \pmb{\varphi} \rangle_{\mathbb{H}}} = 
    {\langle {\widehat{\mathbfcal{C}}}\pmb{\psi}_1, \pmb{\psi}_1 \rangle_{\mathbb{H}}} = \dfrac{\lambda_1}{n-1},$
    where the maximum is taken over all $\pmb{\varphi}$ with $\Vert \pmb{\varphi} \Vert_{\mathbb{H}} = 1$.
    
    \item[iii)] For $2 \leq \ell \leq L$,
    \begin{equation*}
    \max_{\pmb{\varphi}}{\langle {\widehat{\mathbfcal{C}}}\pmb{\varphi}, \pmb{\varphi} \rangle_{\mathbb{H}}} = 
    {\langle {\widehat{\mathbfcal{C}}}\pmb{\psi}_\ell, \pmb{\psi}_\ell \rangle_{\mathbb{H}}} = \dfrac{\lambda_\ell}{n-1},
    \end{equation*}
    where the maximum is taken over all $\pmb{\varphi} \in \mathbb{H}$ with $\Vert \pmb{\varphi} \Vert_{\mathbb{H}} = 1$ and $\langle \pmb{\varphi}, \pmb{\psi}_h \rangle_{\mathbb{H}} = 0$ for $1 \leq h < \ell$.
    
    \item[iv)] The Frobenius norm of the data operator $\mathbfcal{X}$ satisfies:
    \begin{equation*}
    {\Vert \mathbfcal{X} \Vert_\mathbb{F}^2} = \sum_{\ell=1}^L \lambda_\ell.
    \end{equation*}
\end{enumerate}
\end{theorem}
	
Now consider the problem of finding the optimal low-rank approximation of the data operator $\mathbfcal{X}$ with respect to the Frobenius norm, $\Vert . \Vert_\mathbb{F}$. The following result can be viewed as an extension of the Eckart–Young–Mirsky Theorem in the context of functional SVD.

\begin{theorem}\label{Low-rank approximation}
For $\ell=1, \ldots, L$, let $\left(\sqrt{\lambda_\ell}, \pmb{\psi}_\ell, \pmb{u}_\ell\right)$ denote the $\ell^{th}$ singular triplet of $\mathbfcal{X}$ as defined in \eqref{eq-svdl}. Fix $r$ such that $1 \leq r \leq L$. Then, the best rank-$r$ approximation of $\mathbfcal{X}$ in the Frobenius norm satisfies   
\begin{equation*}
    \min_{\{\pmb{w}_\ell, \pmb{\varphi}_\ell\}_{\ell=1}^{r}}{\Vert \mathbfcal{X} - \sum_{\ell=1}^r \pmb{w}_\ell \otimes \pmb{\varphi}_\ell \Vert_\mathbb{F}^2} = \Vert \mathbfcal{X} - \sum_{\ell=1}^r \sqrt{\lambda_\ell} \ \pmb{u}_\ell \otimes \pmb{\psi}_\ell \Vert_\mathbb{F}^2 = \sum_{\ell=r+1}^L \lambda_\ell,
\end{equation*}
where the minimum is taken over all $\pmb{w}_\ell \in \mathbb{R}^n$ and $\pmb{\varphi}_\ell \in \mathbb{H}$, with the functions $\pmb{\varphi}_\ell$ forming an orthonormal set.
\end{theorem}

Both \cite{chiou2014multivariate} and \cite{happ2018multivariate} introduced their MFPCA techniques by solving the following optimization problem:
\begin{equation}
	\label{mfpca1fsvd}
	\max_{\pmb{\varphi}:\Vert \pmb{\varphi} \Vert_{\mathbb{H}} = 1 }{\langle 
			\widehat{\mathbfcal{C}}\pmb{\varphi} , \pmb{\varphi} \rangle_{\mathbb{H}}},
\end{equation}
where $\widehat{\mathbfcal{C}}$ denotes the sample cross-covariance operator. Building upon Theorem \ref{Low-rank approximation}, we propose the following corollary to explore the MFPCA problem through the framework of data operators and the functional SVD approach:
\begin{corollary}\label{col2}
	The solution to \eqref{mfpca1fsvd} and the optimization problem
	\begin{equation} \label{svd eq}
		\min_{\pmb{\varphi}:\Vert\pmb{\varphi}\Vert_{\mathbb{H}}=1,\ \pmb{u}\in\mathbb{R}^n}{\Vert \mathbfcal{X} -  \pmb{u}\otimes \pmb{\varphi} \Vert_\mathbb{F}^2} 
	\end{equation}
	is the same, with the solution given by $\pmb{\varphi} = \pmb{\psi_1}$ and $\pmb{u} = \mathbfcal{X}^*(\pmb{\varphi}).$
\end{corollary}

\section{Regularized MFPCA via Functional SVD} \label{Regularized MFPCA via Functional SVD}
To enhance the performance of MFPCA, regularization techniques have been introduced, yielding more reliable and interpretable results. Traditional approaches to ReMFPCA often concentrate on penalizing the functional components $\pmb{\varphi}$ in \eqref{mfpca1fsvd}. In contrast, our method targets the formulation in \eqref{svd eq}, offering an alternative strategy for penalizing both $\pmb{u}$ and $\pmb{\varphi}$. The regularized functional PCs are estimated by solving the following optimization problem:

\begin{equation}\label{sparse and smooth fpc general}
	\min_{\pmb{u},\pmb{\varphi}:\Vert\pmb{\varphi}\Vert_{\pmb{\alpha}}=1} \Vert \bm{\mathcal{X}} - {\pmb{u}\otimes\pmb{\varphi}}\Vert_\mathbb{F}^2 + \mathcal{P}_{\bm{\theta}}(\pmb{u}, \pmb{\varphi}),
\end{equation}

where $\pmb{u} = (u_1, \ldots, u_n)^\top$, $\pmb{\varphi} = (\varphi_1, \ldots, \varphi_p)$, $\Vert \pmb{\varphi} \Vert_{\pmb \alpha}^2=\sum_{j=1}^p\left( \Vert \varphi_j \Vert_{H_j}^2+ \alpha_j\Vert\mathcal{D}^2\varphi_j\Vert_{H_j}^2\right)$, and $\mathcal{P}_{\bm{\theta}}(\pmb{u}, \pmb{\varphi})$ represents the penalties imposed on both $\pmb{u}$ and $\pmb{\varphi}$. Specifically, we may define:
\begin{equation*}
	\mathcal{P}_{\bm{\theta}}(\pmb{u}, \pmb{\varphi}) = {\pmb{u}}^\top\pmb{u}\sum_{j=1}^p \alpha_j\Vert { \mathcal{D}^2 \varphi_j} \Vert_{H_j}^2 + \sum_{i=1}^{n} p_{\gamma}(\lvert {{u_i}} \rvert),
\end{equation*}
which leads us to the following optimization problem:
\begin{equation}
	\label{optimazation_problem}
	\min_{\pmb{u},\pmb{\varphi}:\Vert\pmb{\varphi}\Vert_{\pmb{\alpha}}=1} \Vert \bm{\mathcal{X}} - {\pmb{u}\otimes\pmb{\varphi}}\Vert_\mathbb{F}^2+
	{\pmb{u}}^\top\pmb{u}\sum_{j=1}^p \alpha_j\Vert { \mathcal{D}^2 \varphi_j} \Vert_{H_j}^2 + \sum_{i=1}^{n} p_{\gamma}(\lvert {{u_i}} \rvert).
\end{equation}
Here, the second term represents the \emph{roughness penalty} imposed on $\pmb{\varphi}$, where the operator $\mathcal{D}^2$ denotes the second derivative, and $\pmb{\alpha} = (\alpha_1,\ldots,\alpha_p)^\top$ is the vector of smoothing tuning parameters. The third term represents the \emph{sparsity penalty function} imposed on $\pmb{u}$, where  $\gamma \geq 0$ is sparse tuning parameter. 

This penalty term differs from the sparse PCA framework of \cite{spca_shen}, which regularizes PC loading vectors to induce variable-level sparsity in finite-dimensional PCA. By contrast, our framework regularizes PC scores to induce subject-level sparsity, so that each functional loading is associated with a relevant subset of subjects and more clearly represents a shared latent pattern. The lemma below is inspired by Lemma 2 of \cite{spca_shen} and adapts their thresholding-based argument to our functional SVD/ReMFPCA framework, where sparsity is imposed on subject-specific scores rather than loading vectors.
\begin{lemma} \label{sparse penalty lemma}
For a fixed $\pmb{\varphi}$, denote $w_i = \langle \pmb{x_{i}},\pmb{\varphi}\rangle_{\mathbb{H}}$, the $\hat{\pmb{u}}$ that minimize \eqref{optimazation_problem} and satisfied $\Vert \pmb{\varphi} \Vert_{\pmb{\alpha}}$ = 1 is as follow:
\begin{enumerate}
    \item Soft thresholding penalty $p_\gamma(|u|) = 2\gamma |u|$ with solution:
    \[\hat{u}_i = h_\gamma^{\text{soft}}(w_i) = \text{sign}(w_i)(|w_i| - \gamma)_+;\] 
    \item Hard thresholding penalty $p_\gamma(|u|) = \gamma^2 I(|u| \neq 0)$ with solution:
    \[\hat{u}_i = h_\gamma^{\text{hard}}(w_i) = I(|w_i| > \gamma) (w_i);\]
    \item SCAD penalty
    \[p_\gamma(|u|) = 
    \begin{cases}
    2\gamma |u| & \text{if } 0 \leq |u| \leq \gamma, \\
    \dfrac{-u^2 + 2a\gamma |u| - \gamma^2}{(a - 1)} & \text{if } \gamma < |u| < a\gamma, \\
    (a + 1)\gamma^2 & \text{if } |u| \geq a\gamma,
    \end{cases}\] with solution:
    \[\hat{u}_i = h_\gamma^{\text{SCAD}}(w_i) = 
    \begin{cases}
    \sign(w_i)(|w_i| - \gamma)_+ & \text{for } |w_i| \leq 2\gamma; \\
    \dfrac{(a - 1)(w_i) - \text{sign}(w_i)a\gamma}{a - 2} & \text{for } 2\gamma < |w_i| \leq a\gamma; \\
    w_i & \text{for } |w_i| > a\gamma.
    \end{cases}\]
   Here, we fix $a = 3.7$ as recommended in \cite{fan2001variable}.
\end{enumerate}
\end{lemma}

For the case of $\gamma$ = 0, which implies no penalty on $\pmb{u}$, functional PCs are estimated from  
\begin{equation}\label{goal_new}
	\min_{\pmb{u},\pmb{\varphi}:\Vert\pmb{\varphi}\Vert_{\pmb{\alpha}}=1} \Vert \bm{\mathcal{X}} - {\pmb{u}\otimes\pmb{\varphi}}\Vert_\mathbb{F}^2 + {\pmb{u}}^\top\pmb{u}\sum_{j=1}^p \alpha_j\Vert { \mathcal{D}^2 \varphi_j} \Vert_{H_j}^2.
\end{equation}
In the following corollary, leveraging the concept of the half-smoothing operator $\mathbfcal{S}_{\pmb{\alpha}}$ as defined in \cite{ReMFPCA}, we reformulate the optimization criterion in \eqref{goal_new} into a quadratic form, which simplifies the optimization process.

\begin{corollary} \label{Corollary SVD ED}
    Let $\mathbfcal{S}_{\pmb{\alpha}}$ denote the half-smoothing operator, as defined in Section~4.1  of \cite{ReMFPCA}. Define $\tilde{\pmb{\varphi}} = \mathbfcal{S}_{\pmb{\alpha}}^{-1}\pmb{\varphi}$ and $\tilde{\mathbfcal{X}} = \mathbfcal{S}_{\pmb{\alpha}}\mathbfcal{X}$, where $\tilde{\pmb{\varphi}} := (\tilde{\varphi}_1, \ldots, \tilde{\varphi}_p) \in \mathbb{H}$ and $\tilde{\mathbfcal{X}} := [\tilde{\pmb{x}}_{i}]_{i=1}^{n} \in \mathbb{F}^{p \times n}$, with $\tilde{\pmb{x}}_{i} := ({\tilde{x}}_{i,1},\ldots,{\tilde{x}}_{i,p}) \in \mathbb{H}$. The criterion in (\ref{goal_new}) is equivalent to:
	\begin{equation*}\label{half_smooth}
		\Vert \mathbfcal{X} \Vert_\mathbb{F}^2 - \Vert \tilde{\mathbfcal{X}} \Vert_\mathbb{F}^2 + \Vert \tilde{\mathbfcal{X}} - \pmb{u} \otimes \tilde{\pmb{\varphi}} \Vert_\mathbb{F}^2.
	\end{equation*}
\end{corollary}

\begin{corollary} \label{col4}
The criterion in (\ref{goal_new}) for estimating smoothed functional PCs is equivalent to the criterion 
     	\begin{equation*} \label{eigenform} 
	\max_{\pmb{\varphi}: \Vert \pmb{\varphi}\Vert_{\pmb{\alpha}} = 1 }\langle{\widehat{\mathbfcal{C}}}\pmb{\varphi},\pmb{\varphi}\rangle_\mathbb{H}.
	\end{equation*}
\end{corollary}
According to Corollary \ref{col4}, our method aligns closely with existing ReMFPCA approaches. The key distinction lies in the application of functional SVD to the data operator for smoothing, rather than relying on the eigen decomposition of the covariance operator. This proposed approach also facilitates the introduction of penalties on the PC scores, allowing the method to adapt to different latent score structures and to reduce the influence of weak or noisy score components through sparsity-inducing regularization.

Building on the approach proposed in \cite{spca_shen}, we can define various types of sparsity penalties suitable for our framework. The Lemma \ref{sparse penalty lemma}  introduces three types of penalties specifically tailored for estimating the optimal sparse $\hat{\pmb{u}}$ in the context of ReMFPCA.
\vspace{-.12in}

\section{ReMFPCA via Power Algorithm}
\label{Power_Algorithm}
In this section, we detail the implementation of the proposed functional SVD of the data operator using extensions of the power algorithm. 
To make computations feasible, we approximate the infinite-dimensional Hilbert space $H_j$ by a finite-dimensional subspace spanned by a chosen set of basis functions, 
$H_j \approx \text{span}\{\phi_j^k\}_{k=1}^{D_j}$. 
Therefore, each element $x_{i,j}(t) \in H_j$ can be expressed as
\[
x_{i,j}(t) = \pmb{\phi}_j(t)^\top \pmb{c}_{i,j},
\]
where $\pmb{c}_{i,j} = (c_{i,j,1}, \ldots, c_{i,j,D_j})^\top \in \mathbb{R}^{D_j}$, 
$\pmb{\phi}_j = (\phi_j^1, \ldots, \phi_j^{D_j})^\top$, $t \in \mathcal{T}_j$, and $i = 1, \ldots, n$.
 Let \(\pmb{v}\) represent the coefficients corresponding to \(\pmb{\varphi}\) as defined in \eqref{sparse and smooth fpc general}. Specifically, for \(\pmb{\varphi} = (\varphi_1, \ldots, \varphi_p)\), we define the coefficient vector \(\pmb{v} \in \mathbb{R}^D\) with \(D = \sum_{j=1}^{p} D_j\) as
\[
  \pmb{v} = \bigl(\pmb{v}_1^\top, \ldots, \pmb{v}_p^\top\bigr)^\top,
\]
where each \(\pmb{v}_j \in \mathbb{R}^{D_j}\), and \(\varphi_j\) is given by
\[
  \varphi_j = \pmb{\phi}_j(t)^\top \pmb{v}_j.
\]
The remainder of this section is dedicated to developing two extensions of the power algorithm and selecting the corresponding tuning parameters.
\subsection{Sequential Method}\label{Sequential_Method}
The functional SVD technique's ability to perform sequential rank-one approximations allows for the stepwise extraction of functional PCs. This approach enables the application of distinct tuning parameters to each PC, thereby enhancing the accuracy and relevance of the estimations \citep{huang2008functional}. Utilizing this finite-dimensional representation, the optimization problem in \eqref{optimazation_problem} can be reformulated in matrix form as:
\begin{equation} \label{sparse_fpca_003}
	\min_{\pmb{u}, \pmb{v}} \  \left\{\operatorname{tr}(\pmb{C}\pmb{G}\pmb{C}^{\top}) - 2\pmb{u}^{\top}\pmb{C}\pmb{G}\pmb{v} + \pmb{u}^{\top}\pmb{u}\,{\pmb{v}^{\top} (\pmb{G} + \pmb{D}_{\pmb\alpha}) \pmb{v}} + \sum_{i=1}^{n} p_{\gamma}(\lvert {{u_i}} \rvert)\right\}.
\end{equation}
Here $\pmb{G} = \diag\{\pmb{G}_1, \ldots, \pmb{G}_p \}$, where $\pmb{G}_j = \left[\langle \phi_{j}^{l^{*}} , \phi_{j}^{k^{*}} \rangle_{H_j}\right]_{l^{*},k^{*}=1}^{D_j}$ denotes the Gram matrix of the space $H_j$. The matrix $\pmb{D}_{\pmb\alpha}$ represents a roughness penalty associated with the tuning parameter $\pmb\alpha$, while $\pmb{C}$ is the coefficient matrix defined as $\pmb{C}:=\begin{bmatrix}\pmb{c}_1&\ldots&\pmb{c}_n\end{bmatrix}^\top$, with $\pmb{c}_{i} = (\pmb{c}_{i,1}^\top, \ldots, \pmb{c}_{i,p}^\top)^\top$. By defining $\tilde{\pmb{C}} = \pmb{C}\pmb{G}^{\frac{1}{2}}$, $\tilde{\pmb{v}} = \pmb{G}^{\frac{1}{2}}\pmb{v}$, and ${\pmb{\Omega_\alpha}} = \pmb{G}^{-\frac{1}{2}}\pmb{D_\alpha}\pmb{G}^{-\frac{1}{2}}$, the optimization problem in \eqref{sparse_fpca_003} is equivalent to the following optimization criterion:
\begin{equation} \label{sparse fpca 004}
    \min_{\pmb{u}, \tilde{\pmb{v}}} \  \left\{\Vert\tilde{\pmb{C}} - \pmb{u}\tilde{\pmb{v}}^{\top}\Vert_{F}^2 + \pmb{u}^{\top}\pmb{u}\tilde{\pmb{v}}^{\top}{\pmb{\Omega_\alpha}}\tilde{\pmb{v}} + \sum_{i=1}^{n} p_{\gamma}(\lvert {{u_i}} \rvert)\right\}.
\end{equation}
Let $\pmb{S}_{\pmb{\alpha}} = (\pmb{G} + \pmb{D}_{\pmb\alpha})^{-\frac{1}{2}}$. The minimization of \eqref{sparse fpca 004} can be achieved through an iterative process consisting of two alternating steps, followed by normalizing either $\pmb{u}$ or $\tilde{\pmb{v}}$:

\begin{enumerate}
    \item Fixing $\tilde{\pmb{v}}$, set $\pmb{u} = h_{\gamma}(\tilde{\pmb{C}}\tilde{\pmb{v}})$, where $h_{\gamma}$ is defined in Lemma \ref{sparse penalty lemma}.
    \item Fixing $\pmb{u}$, set $\tilde{\pmb{v}} = \pmb{G}^{\frac{1}{2}}\pmb{S}_{\pmb{\alpha}}^{2}\pmb{G}^{\frac{1}{2}}{\tilde{\pmb{C}}}^{\top}\pmb{u}$.
\end{enumerate}

We define $\pmb{G}^{\frac{1}{2}}\pmb{S}_{\pmb{\alpha}}^{2}\pmb{G}^{\frac{1}{2}}$ as $\tilde{\pmb{S}}_{\pmb{\alpha}}$, which can be interpreted as a smoothing matrix that operates on $\tilde{\pmb{C}}$. Notably, when both penalties are absent (i.e., $\pmb{\alpha} = \pmb{0}$ and $\gamma = 0$), the algorithm reduces to the power method for SVD. Building on this framework, we introduce Algorithm \ref{algorithm4.1}, a sequential method that extracts functional PCs one at a time, allowing for flexible tuning. Algorithm~\ref{algorithm4.1} can be viewed as a penalized power iterative procedure. When only the roughness penalty is imposed, Corollary~\ref{Corollary SVD ED} shows that the problem is equivalent to an unpenalized functional SVD formulation after the half-smoothing transformation. In this sense, the resulting iteration is closely related to the classical power method and inherits its standard convergence intuition. When the sparsity-inducing penalty on the scores is also included, however, a full formal convergence analysis is no longer straightforward. Therefore, we do not pursue a rigorous convergence theorem for the general penalized sequential algorithm here. In our numerical studies, the algorithm exhibited stable convergence behavior, similar to the convergence behavior reported in \cite{spca_shen} and \cite{allen}.

\begin{figure}[t]
    \centering
        \begin{algorithm}[H]
            \caption{Sequential penalized power algorithm for estimating regularized multivariate functional principal components with roughness penalties on the functional PCs and optional sparsity penalties on the PC scores.\label{algorithm4.1}}
            \begin{algorithmic}[1]
                \State Initialize ${\tilde{\pmb{v}}} \in \mathbb{R}^{D}$
                \State Repeat the following steps until convergence:
                    	\begin{enumerate}[(a)]
                        \item $\pmb{u}$ $\leftarrow$ $h_{\gamma}(\tilde{\pmb{C}}\tilde{\pmb{v}})$, where $h_{\gamma}$ is defined on Lemma \ref{sparse penalty lemma}, followed by normalization.
                            \item $\tilde{\pmb{v}}$ $\leftarrow$ $\tilde{\pmb{S}}_{\pmb{\alpha}}{\tilde{\pmb{C}}}^{\top}\pmb{u}$.
                        \end{enumerate}
                \State Obtain $\pmb{v}$ by computing $\pmb{G}^{-\frac{1}{2}}\tilde{\pmb{v}}$ and normalize.
                \vspace{0.5mm}
                \State Update $\tilde{\pmb{C}} = (\pmb{I}-\pmb{u}{\pmb{u}}^{\top})\tilde{\pmb{C}}$ and proceed to find the next functional PC.
            \end{algorithmic}
        \end{algorithm}
\end{figure}

\subsection{Tuning Parameters Selection} \label{Tuning Parameters Selection}
In the sequential power algorithm, jointly considering both tuning parameters presents a considerable computational challenge due to the need for a search process over a multi-dimensional grid of regularization parameters. Drawing inspiration from \cite{spfca}, which proposed a two-step procedure for sequentially selecting the tuning parameters, we develop a similar approach. Our method, a two-step conditional tuning parameter selection process, is outlined as follows:

\textbf{Step 1:} In this step, we ignore roughness penalty and use K-fold Cross Validation (CV) criterion to select the optimal $\gamma$. In this CV criterion, we draw inspiration from the approach of \cite{spca_shen}, specifically in using the degree of sparsity of $\pmb{u}$ as the tuning parameter. This approach serves two primary purposes: enhancing interpretability and providing computational benefits. We define the degree of sparsity as an integer $i \in \{1, \ldots, n\}$, corresponding to setting $\gamma = |\tilde{\pmb{C}}\tilde{\pmb{v}}|_{(i)}$, where $|\tilde{\pmb{C}}\tilde{\pmb{v}}|_{(i)}$ denotes the $i^{th}$ order statistic of $|\tilde{\pmb{C}}\tilde{\pmb{v}}|$. The CV selection algorithm is detailed in Algorithm \ref{algorithm6}.

\begin{figure}[!b]
    \centering
\begin{algorithm}[H]
    \caption{K-fold cross-validation procedure for selecting the sparsity tuning parameter in the sequential ReMFPCA algorithm. \label{algorithm6}}
    \begin{algorithmic}[1]
        \State Randomly partition the columns of $\tilde{\pmb{C}}$ into $K$ roughly equal-sized groups, denoted as $\tilde{\pmb{C}}^{1}, \dots, \tilde{\pmb{C}}^{K}$.
        \State For each candidate tuning parameter $\gamma_h$ among the CV choices, do the following:
        \begin{enumerate}[(a)]
            \item For $k = 1, \dots, K$, define $\tilde{\pmb{C}}^{-k}$ as the matrix formed by leaving out $\tilde{\pmb{C}}^{k}$.
            \item Apply \emph{Algorithm \ref{algorithm4.1}} on $\tilde{\pmb{C}}^{-k}$ with $\pmb\alpha = \pmb{0}$ and $\gamma = \gamma_{h^{*}}$ to obtain the functional PC scores $\pmb{u}^{-k}$. 
            \item Project $\tilde{\pmb{C}}^{k}$ onto $\pmb{u}^{-k}$ to calculate the projection coefficients $\pmb{v}^k$.
            \item Calculate the K-fold CV score as:
            \begin{equation*}
                CV_h = \sum_{k=1}^{K} \frac{\Vert \tilde{\pmb{C}}^{k} - \pmb{u}^{-k} {(\pmb{v}^k)}^{\top} \Vert^2}{N_k},
            \end{equation*}
            where $N_k$ is the number of elements in $\tilde{\pmb{C}}^{k}$.
        \end{enumerate}
        \State Select the optimal tuning parameter $\gamma_{h^{*}}$ as the one with the smallest CV score.
    \end{algorithmic}
\end{algorithm}

\end{figure}

\textbf{Step 2:} After selecting the optimal $\gamma$ and obtaining sparse $\pmb{u}$, a closed-form CV or GCV criterion is derived to select the optimal $\pmb{\alpha}$, given by:
\begin{align}\label{CV score}
CV_{\pmb{\alpha}} 
&= \frac{1}{D}\sum_{j=1}^{p}{\sum_{d=1}^{D_j}{\dfrac{\Vert\{(\pmb{I}-{\tilde{\pmb{S}}_{{\alpha}_j}})(\tilde{\pmb{C}}_{j}^{\top}\pmb{u})\}_{d}\Vert^2}{(1-\{\tilde{\pmb{S}}_{{\alpha}_j}\}_{dd})^2}}},
\end{align}
\begin{equation} \label{GCV score}
\begin{split}
\MoveEqLeft
GCV_{\pmb{\alpha}} = \frac{1}{D}\sum_{j=1}^{p}{{\dfrac{\Vert(\pmb{I}-{\tilde{\pmb{S}}_{{\alpha}_j}})(\tilde{\pmb{C}}_{j}^{\top}\pmb{u})\Vert^2}{(1-\frac{1}{D_j}\tr\{\tilde{\pmb{S}}_{{\alpha}_j}\})^2}}},
\end{split}
\end{equation}
where $\tilde{\pmb{S}}_{{\alpha}_j}$ is the $j^{th}$ diagonal block of $\tilde{\pmb{S}}_{\pmb{\alpha}}$ and $\tilde{\pmb{C}}_{j}$ refers to the submatrix of $\tilde{\pmb{C}}$ corresponding to the $j^{th}$ variable, denoted by $\begin{bmatrix}\pmb{G}_j\pmb{c}_{1,j}&\ldots&\pmb{G}_j\pmb{c}_{n,j}\end{bmatrix}^\top$. Note that if only one penalty is imposed, we simply skip the step corresponding to the penalty that is not applied.

The derivation of these closed-form smoothing criteria, equations \eqref{CV score} and \eqref{GCV score}, is inspired by the approach outlined in \cite{huang2008functional}, where the estimation of $\tilde{\pmb{v}}$ with a fixed $\pmb{u}$ is treated as a penalized least squares problem, with $\tilde{\pmb{v}}$ interpreted as the coefficient vector. In our power algorithm, once the sparse vector $\pmb{u}$ is estimated, $\tilde{\pmb{v}}$ is updated in step 2(b) while keeping $\pmb{u}$ fixed, simplifying equation \eqref{sparse fpca 004} to:

\begin{equation*}
    \min_{\tilde{\pmb{v}}} \left\{\Vert\tilde{\pmb{C}} - \pmb{u}\tilde{\pmb{v}}^{\top}\Vert_{F}^2 + \pmb{u}^{\top}\pmb{u}\tilde{\pmb{v}}^{\top}\pmb{\Omega_\alpha}\tilde{\pmb{v}}\right\}.
\end{equation*}

Let the response vector $\bar{\pmb{y}}$ be defined as:

\begin{equation*}
    \bar{\pmb{y}}:= \begin{pmatrix}
        \tilde{\pmb{C}}_{.1}^{\top},
        \tilde{\pmb{C}}_{.2}^{\top},
        \dots,
        \tilde{\pmb{C}}_{.D}^{\top}
    \end{pmatrix}^{\top} \in \mathbb{R}^{nD},
\end{equation*}

where $\tilde{\pmb{C}}_{.j}$ represents the $j^{\text{th}}$ column of $\tilde{\pmb{C}}$. The design matrix is then given by $\bar{\pmb{X}} := \diag\{\pmb{u}, \ldots, \pmb{u}\} \in \mathbb{R}^{nD\times D}$. This formulation leads to the following penalized least squares equation:

\begin{equation}\label{pen ls eq}
    \Vert\bar{\pmb{y}}-\bar{\pmb{X}}\tilde{\pmb{v}}\Vert^2 + \tilde{\pmb{v}}^{\top}\bar{\pmb{\Omega}}_{\pmb{\alpha}}\tilde{\pmb{v}},
\end{equation}

where $\bar{\pmb{\Omega}}_{\pmb{\alpha}} := \Vert\pmb{u}\Vert^2 \pmb{\Omega_\alpha}$.

We define $\hat{\tilde{\pmb{v}}}$ as the estimator of $\tilde{\pmb{v}}$ that minimizes \eqref{pen ls eq}. Let $\hat{\tilde{\pmb{v}}}^{(-j)}$ represent the estimate obtained by omitting the $j^{\text{th}}$ block of $\bar{\pmb{y}}$ and the corresponding rows of $\bar{\pmb{X}}$. When constructing functions using local basis functions, such as the B-spline basis, this omission is equivalent to the removal of local information. This method is analogous to the cross-validation approach described in \cite{huang2008functional}, which also involves the removal of local information (grid points) in each iteration. We now introduce the following lemma regarding the cross-validation prediction errors.

\begin{lemma}\label{CVlemma}
    Let {$\hat{\tilde{v}}_{d}$} denotes the $d^{th}$ element of $\hat{\tilde{\pmb{v}}} = \frac{\tilde{\pmb{S}}_{\pmb{\alpha}}\tilde{\pmb{C}}^{\top}\pmb{u}}{{\pmb{u}}^{\top}\pmb{u}}$. The sum of squares of the prediction error for the $d^{th}$ leave-one-out cross-validation is given by:
    \begin{equation}\label{00cv}
    \Vert\hat{{\tilde{v}}}^{(-d)}_{d}\pmb{u}-\tilde{\pmb{C}}_{.d}\Vert^2 = \tilde{\pmb{C}}_{.d}^{\top}\tilde{\pmb{C}}_{.d} - \dfrac{(\tilde{\pmb{C}}_{.d}^{\top}\pmb{u})^2}{\Vert\pmb{u}\Vert^2} + \dfrac{\left(\Vert\pmb{u}\Vert\hat{{\tilde{v}}}_{d} - \frac{\tilde{\pmb{C}}_{.d}^{\top}\pmb{u}}{\Vert\pmb{u}\Vert}\right)^2}{(1-\lambda^{*}_{d}\Vert\pmb{u}\Vert^2)^2},
    \end{equation}
    where $\lambda^{*}_{d} = \dfrac{{\{\tilde{\pmb{S}}_{\pmb{\alpha}}\}}_{dd}}{\Vert\pmb{u}\Vert^2}$.
\end{lemma}

Given that $\tilde{\pmb{C}}$ and $\pmb{u}$ are fixed, equations \eqref{CV score} and \eqref{GCV score} are derived by averaging the final terms on the right-hand side of equation \eqref{00cv}.

\subsection{Joint Method for Smoothing}
Alongside the sequential power algorithm introduced in subsection \ref{Sequential_Method}, if the goal is to extract \(Q\) smooth functional PCs jointly rather than sequentially, the optimization problem in \eqref{goal_new} can be reformulated in matrix form as:
\begin{equation} \label{fpca_joint}
	\min_{\pmb{U}, \pmb{V}} \  \left\{\operatorname{tr}(\pmb{C}\pmb{G}\pmb{C}^{\top}) - 2\operatorname{tr}(\pmb{U}^{\top}\pmb{C}\pmb{G}\pmb{V}) + \operatorname{tr}(\pmb{U}^{\top}\pmb{U}\,{\pmb{V}^{\top} (\pmb{G} + \pmb{D}_{\pmb\alpha}) \pmb{V}}) \right\},
\end{equation}
where \(\pmb{U}\) and \(\pmb{V}\) are matrices with columns \(\pmb{u}_{q}\) and \(\pmb{v}_{q}\) for \({q} \in \{1, \ldots, Q\}\).

Similar to subsection \ref{Sequential_Method}, the optimization problem in \eqref{fpca_joint} can be solved jointly to extract multiple functional PCs together using the following optimization criterion:
\begin{equation} \label{fpca_joint_criterion}
    \min_{\pmb{U}, \tilde{\pmb{V}}} \  \left\{\Vert\tilde{\pmb{C}} - \pmb{U}\tilde{\pmb{V}}^{\top}\Vert_{F}^2 + \tr(\pmb{U}^{\top}\pmb{U}\tilde{\pmb{V}}^{\top}{\pmb{\Omega_\alpha}}\tilde{\pmb{V}})\right\},
\end{equation}
where $\tilde{\pmb{V}}= \begin{bmatrix}\tilde{\pmb{v}}_1 & \cdots & \tilde{\pmb{v}}_{Q}\end{bmatrix}$ such that $\tilde{\pmb{v}}_{q} = \pmb{G}^{\frac{1}{2}}\pmb{v}_{q}$. Next, we introduce Algorithm \ref{Joint Power Algorithm}, a joint power algorithm designed to estimate functional PCs concurrently. Unlike traditional sequential methods that extract components one at a time, our joint approach processes all components simultaneously.

\begin{figure}[!t]
    \centering
        \begin{algorithm}[H]
            \caption{Joint power algorithm for simultaneous estimation of multiple smoothed multivariate functional principal components. \label{Joint Power Algorithm}}
            \begin{algorithmic}[1]
                \State Initialize ${\tilde{\pmb{V}}} \in \mathbb{R}^{D\times Q}$, where $Q$ denotes the number of functional PCs estimated jointly.
                \State Repeat the following steps until convergence:
                    	\begin{enumerate}[(a)]
                        \item $\pmb{U}$ $\leftarrow$ $\tilde{\pmb{C}}\tilde{\pmb{V}}$, where $\pmb{U} \in  \mathbb{R}^{n\times Q}$, followed by columnwise normalization.
                         \item Obtain $\pmb{Q}$ as the orthogonal matrix from the QR decomposition of ${\tilde{\pmb{C}}}^{\top}\pmb{U}.$
                            \item $\tilde{\pmb{V}}$ $\leftarrow$ $\tilde{\pmb{S}}_{\pmb{\alpha}}\pmb{Q}$.
                        \end{enumerate}
                \State Obtain $\pmb{V}$ by computing $\pmb{G}^{-\frac{1}{2}}\tilde{\pmb{V}}$ and normalize.
            \end{algorithmic}
        \end{algorithm}
\end{figure}

Using a similar approach as in deriving \eqref{CV score} and \eqref{GCV score}, we derive the following closed-form CV and GCV criteria to select the optimal $\pmb{\alpha}$:
\begin{align}\label{CV score2}
CV_{\pmb{\alpha}} 
&= \frac{1}{D}\sum_{j=1}^{p}{\sum_{d=1}^{D_j}{\dfrac{\Vert\{(\pmb{I}-{\tilde{\pmb{S}}_{{\alpha}_j}})(\tilde{\pmb{C}}_{j}^{\top}\pmb{U})\}_{d.}\Vert^2}{(1-\{\tilde{\pmb{S}}_{{\alpha}_j}\}_{dd})^2}}},
\end{align}
\begin{equation} \label{GCV score2}
\begin{split}
\MoveEqLeft
GCV_{\pmb{\alpha}} = \frac{1}{D}\sum_{j=1}^{p}{{\dfrac{\Vert(\pmb{I}-{\tilde{\pmb{S}}_{{\alpha}_j}})(\tilde{\pmb{C}}_{j}^{\top}\pmb{U})\Vert^2}{(1-\frac{1}{D_j}\tr\{\tilde{\pmb{S}}_{{\alpha}_j}\})^2}}}.
\end{split}
\end{equation}
The derivation of the joint GCV criterion is detailed in Appendix \ref{joint GCV proof}. While the results obtained from the joint power algorithm are closely aligned with those from existing ReMFPCA methods, such as \cite{ReMFPCA} and \cite{kayano2009functional}, the proposed closed-form CV and GCV criteria provide a significant advantage. Specifically, they can be seamlessly integrated into these methods to greatly enhance computational efficiency.

 \section{Percentage of Variability}\label{sect::variability}
In MFPCA, the functional PCs are uncorrelated, and their corresponding PC scores are orthogonal. This orthogonality ensures that the variance explained by each functional PC directly corresponds to the variance of its associated PC scores. Consequently, the total variance explained by the first \( k \) functional PCs is simply the sum of the variances of their respective PC scores, leading to a cumulative measure of explained variance.

For functional PCs estimated using the joint power algorithm with only a roughness penalty, the estimated functional PCs themselves may not be orthogonal. However, Corollary~\ref{Corollary SVD ED} shows that, when only the roughness penalty is imposed, the problem is equivalent, after the half-smoothing transformation, to an unpenalized functional SVD of the transformed data operator $\tilde{\mathbfcal{X}}$. In that equivalent formulation, the associated score vectors are the right singular vectors of $\tilde{\mathbfcal{X}}$, and hence are orthogonal by the standard orthogonality property of singular value decomposition. This orthogonality ensures no overlap of information between the functional PCs, allowing the variance explained by each functional PC to be estimated in the same manner as in standard MFPCA.

When a sparsity penalty is imposed, or even when only a roughness penalty is applied in the sequential power algorithm, these orthogonality properties in the PC scores are lost, leading to some overlap in the information captured by the functional PCs. \cite{spca_shen} discusses methods for computing explained variance under these overlap conditions in sparse PCA, and these approaches can be adapted to our regularized MFPCA framework.

Consider the set \(\pmb{\Psi} = \text{span}\{\pmb{\psi}_\ell\}_{\ell=1}^r \subseteq \mathbb{H}\), where each \(\pmb{\psi}_\ell\) represents the \(\ell^{th}\) estimated functional PC. The projection operator of the data \(\bm{\mathcal{X}}\) onto \(\pmb{\Psi}\) is defined as:

\[
\bm{\mathcal{X}}^{(r)} = \sum_{\ell=1}^r \tilde{\pmb{u}}_\ell \otimes \pmb{\psi}_\ell,
\]

where \(\tilde{\pmb{u}}_\ell\) denotes the \(\ell^{th}\) column of the matrix \(\tilde{\pmb{U}}\), which is given by:

\[
\tilde{\pmb{U}} := 
\begin{bmatrix}
\langle \pmb{x}_1, \pmb{\psi}_1 \rangle_\mathbb{H} & \cdots & \langle \pmb{x}_1, \pmb{\psi}_r \rangle_\mathbb{H} \\
\langle \pmb{x}_2, \pmb{\psi}_1 \rangle_\mathbb{H} & \cdots & \langle \pmb{x}_2, \pmb{\psi}_r \rangle_\mathbb{H} \\
\vdots & \ddots & \vdots \\
\langle \pmb{x}_n, \pmb{\psi}_1 \rangle_\mathbb{H} & \cdots & \langle \pmb{x}_n, \pmb{\psi}_r \rangle_\mathbb{H}
\end{bmatrix} \pmb{W}^{-1}.
\]

Here, \(\pmb{W}\) represents the Gram matrix associated with \(\pmb{\Psi}\), defined as \(\pmb{W}_{ij} = \langle \pmb{\psi}_i, \pmb{\psi}_j \rangle_\mathbb{H}\).

The total variance explained by the first \( r \) functional PCs can be defined as 
\[
\Vert \bm{\mathcal{X}}^{(r)} \Vert_\mathbb{F}^2 = \sum_{i=1}^n \langle \bm{x}^{(r)}_i, \bm{x}^{(r)}_i \rangle_\mathbb{H},
\]
where \(\bm{x}^{(r)}_i\) is the \(i^{\text{th}}\) observation in \(\bm{\mathcal{X}}^{(r)}\). Notably, if \(\bm{\Psi}\) consists of the functional PCs derived from conventional MFPCA, then \(\bm{\mathcal{X}}^{(r)}\) aligns with the representation described in Theorem \ref{th-fsvd}. This alignment implies that our definition of the total variance explained is consistent with the traditional definition used in MFPCA. The following theorem establishes that the newly defined variance increases as additional basis functions \(\bm{\psi}_j\) are incorporated and is bounded above by the total variance of the data operator \(\bm{\mathcal{X}}\), calculated as 
\[
\Vert \bm{\mathcal{X}} \Vert_\mathbb{F}^2 = \sum_{i=1}^n \langle \bm{x}_i, \bm{x}_i \rangle_\mathbb{H}.
\]

\begin{theorem}\label{adjusted variance explained}
$
\Vert \bm{\mathcal{X}}^{(r)} \Vert_\mathbb{F}^2 \le \Vert \bm{\mathcal{X}}^{(r+1)} \Vert_\mathbb{F}^2 \le \Vert \bm{\mathcal{X}} \Vert_\mathbb{F}^2.
$
\end{theorem}	

Based on Theorem \ref{adjusted variance explained}, the adjusted variance of the \(r^{\text{th}}\) functional PC is estimated as 
$\Vert \bm{\mathcal{X}}^{(r)} \Vert_\mathbb{F}^2 - \Vert \bm{\mathcal{X}}^{(r-1)} \Vert_\mathbb{F}^2.$ The cumulative percentage of explained variance (CPEV) by the first \(r\) PCs is then calculated as $\Vert \bm{\mathcal{X}}^{(r)} \Vert_\mathbb{F}^2 / \Vert \bm{\mathcal{X}} \Vert_\mathbb{F}^2$

\section{Simulation Study}
\label{Simulation}
In this section, we demonstrate the performance of our approach through simulation. We begin by introducing a bivariate functional object denoted as $\pmb{X}(\pmb{t})$, which can be represented as $({X_1(t_1)},{X_2(t_2)})^{\top}$, and an orthonormal basis system $\pmb{\psi}_m (\pmb{t})$ in the the interval [0,1] defined as $(\psi_m^{(1)}(t_1),$ $\psi_m^{(2)}(t_2))^\top$. The basis functions $\psi_m^{(1)}(t)$ and $\psi_m^{(2)}({t})$ are given by $\sin \left((2m-1) \pi t\right)$ and $\sin \left(\frac{(4m-3)\pi}{2}t\right)$, respectively. In all simulation settings, each functional trajectory was discretized at 100 time points sampled on a regular grid over $[0,1]$. In the following two subsections, we consider two cases: a non-sparse case that addresses only roughness in the functional PCs and a sparse case that addresses both roughness in the functional PCs and sparsity in the PC scores.

\subsection{Non-sparse Case}
In this setup, we impose only a roughness penalty and adopt the following functional data generating model:
\begin{equation} \label{generatesvd}
	\pmb{X}_i(\pmb{t}) = \sum_{m=1}^{M} \rho_{i,m} \,\tilde{\pmb{\psi}}_m (\pmb{t}), \qquad i=1,\ldots,n,
\end{equation}
where $\rho_{i,m}$ are i.i.d. normal random variables with mean $0$ and variance $\lambda_m$, and the perturbed eigenfunctions are given by
\[
\tilde{\pmb{\psi}}_m (\pmb{t}) = \pmb{\psi}_m (\pmb{t}) + \pmb{\epsilon}_{m}(\pmb{t}).
\]

The measurement noise $\pmb{\epsilon}_{m}(\pmb{t})$ follows an i.i.d. normal distribution with mean $\pmb{0}$ and covariance matrix
\[
\pmb\Sigma_m =
\begin{bmatrix}
\sigma_m^2 & \rho\sigma_m^2  \\
\rho\sigma_m^2 & \sigma_m^2
\end{bmatrix},
\]
where $\sigma_m^2$ denotes the variance of the measurement noise for the first and second variables, and $\rho$ represents the correlation between them. By allowing each component $m$ to have its own noise level $\sigma_m^2$, we create heterogeneous conditions across the principal components. This setting highlights the flexibility of our sequential power algorithm, which assigns distinct tuning parameters to different PCs in response to their noise levels. For completeness, we also consider a more general case with uniform noise levels across all components, and show that our method remains competitive even in this simplified scenario.

In the generative model \eqref{generatesvd}, we considered \( M = 4 \) and explored two eigenvalue decay patterns: linear (\(\lambda_m^{\text{lin}} = \frac{2M+1-2m}{2M-1}\)) and exponential (\(\lambda_m^{\text{exp}} = \exp(-m/2)\)). The purpose of considering both linear and exponential decay patterns is to evaluate the performance of the proposed method under different rates of eigenvalue decay. The correlation coefficient was fixed at \(\rho = 0.4\). For the scenario involving varying degrees of roughness among the PCs, the variances \(\sigma_{m}^2\) were set at 0.1, 0.05, 0.025, and 0.01, representing PCs with roughness ranging from high to low. In contrast, for the uniform roughness scenario, a constant variance of \(\sigma_{m}^2 = 0.05\) was applied for all \(m\). We conducted 500 simulations for sample sizes \(N = 200, 500, 1000, 2000\) observations, based on the aforementioned setups.

We compare our sequential and joint smooth MFPCA with two competitive MFPCA approaches that incorporate smoothing effects. The first approach is MFPCA applied after data smoothing, and the second is the method proposed by \cite{happ2018multivariate}, which achieves smoothing by employing the algorithm from \cite{PACE} \cite[see details in the simulation section of][]{happ2018multivariate}. Additionally, we include a comparison with non-regularized MFPCA, serving as a baseline.

\begin{figure}[!t] 
\centering
  \includegraphics[width=1\linewidth]{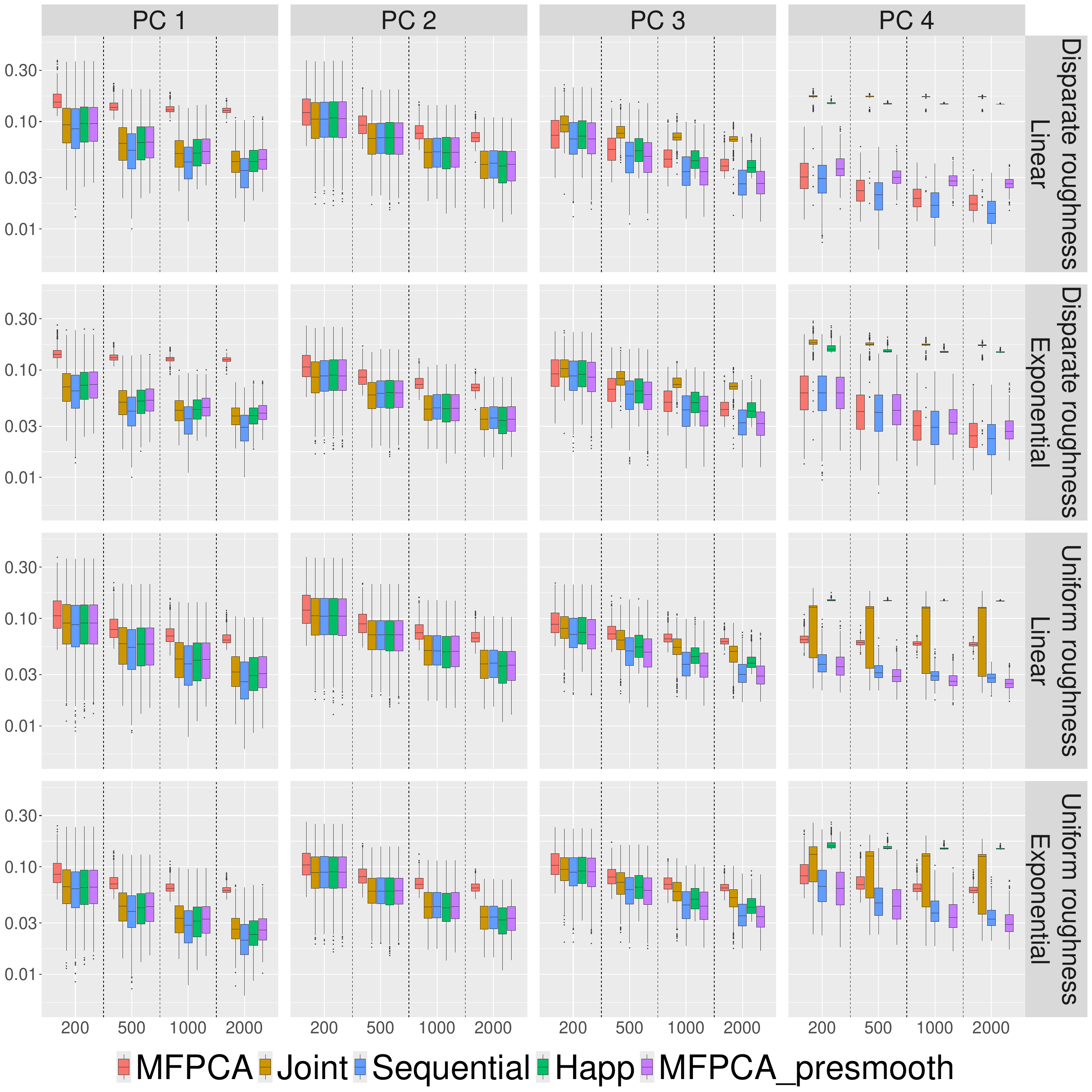}
  \caption{Estimation error of the functional principal components in the non-sparse simulation setting. The error is measured by 
$Err(\hat{\pmb\psi}_m)=\Vert \hat{\pmb\psi}_m-\pmb\psi_m\Vert_{\mathbb{H}}$ 
for the first four functional PCs across the competing MFPCA and ReMFPCA methods.}
  \label{pc error}
\end{figure}

The evaluations were conducted using two metrics: $Err(\hat{\pmb\psi}_m) = \Vert\hat{\pmb\psi}_m-\pmb\psi\Vert_{\mathbb{H}}$ and the mean relative absolute error (MRAE), defined as $\frac{1}{n}\sum_{i=1}^{n} \frac{\Vert\hat{\pmb{x}}_i-\pmb{x}_i\Vert_{\mathbb{H}}}{\Vert\pmb{x}_i\Vert_{\mathbb{H}}}$, where $\hat{\pmb{x}}_i = \sum_{m=1}^M \hat{\rho}_{i,m}\hat{\pmb\psi}_m$ represents the reconstructed function, with $\hat{\rho}_{i,m}$ as the estimated component score. The performance of each method, illustrated in Figures \ref{pc error} and \ref{mare 2} for the first four functional PCs, demonstrates the effectiveness of our GCV criterion in selecting optimal smoothing parameters (see additional performance evaluation using $Err(\hat{\lambda}_m) = \frac{|\hat{\lambda}_m-\lambda_m|}{|\lambda_m|}$ in Figure \ref{eigenvalue error}). Furthermore, our sequential approach surpasses other methods when addressing PCs with varying levels of roughness and performs competitively when the roughness of the PCs is consistent. Although the joint power method produces competitive results for the first two PCs across scenarios of both varying and uniform roughness, its reliance on a single tuning parameter limits its ability to effectively smooth the last two PCs. We also note that the pre-smoothing-based MFPCA method also performs competitively, particularly for the later components, such as PC3 and PC4. This suggests that pre-smoothing can be effective when roughness in the observed functional data is the main source of estimation difficulty.

\begin{figure}[!t] 
\centering
  \includegraphics[width=1\linewidth]{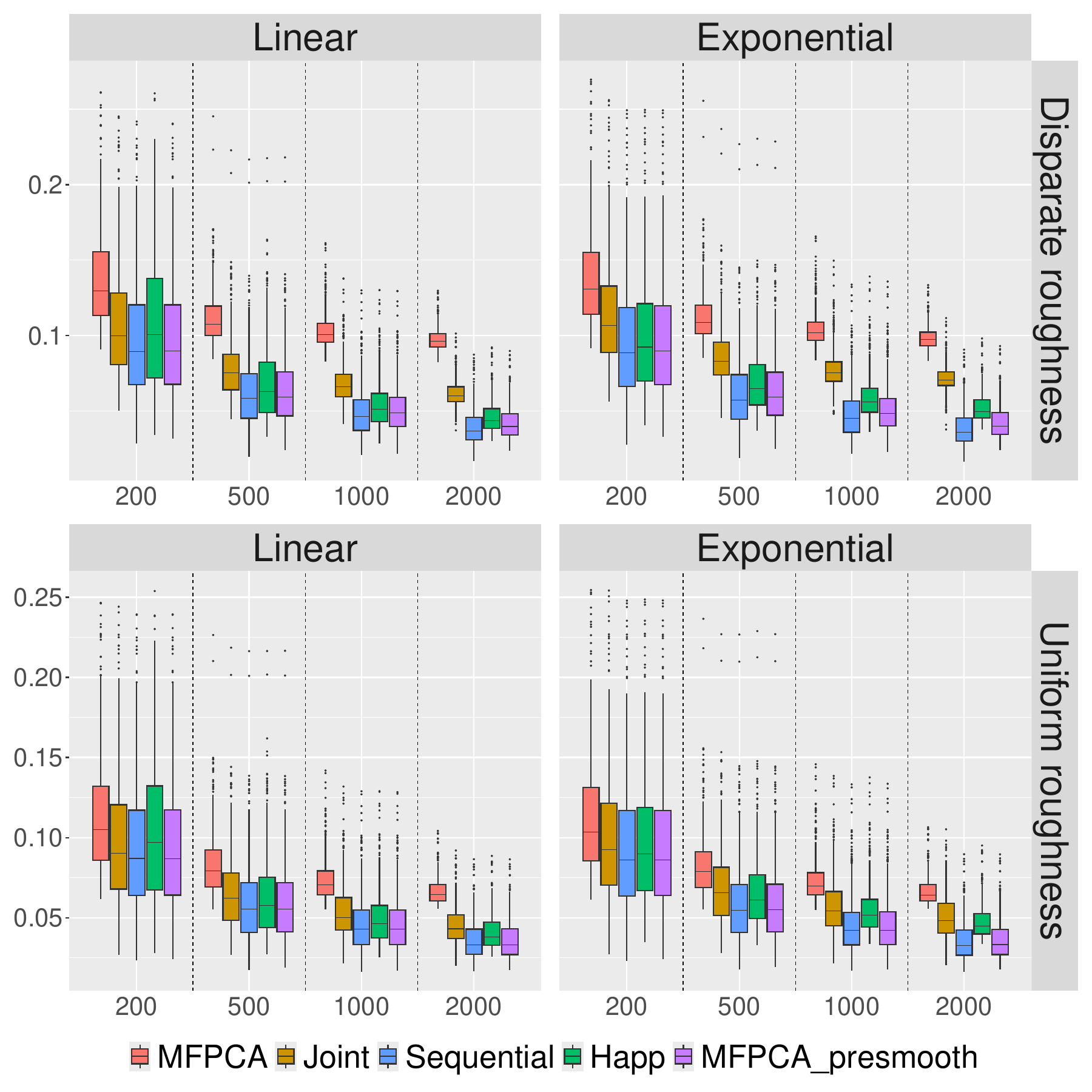}
  \caption{Mean relative absolute reconstruction error (MRAE) in the non-sparse simulation setting across the competing MFPCA and ReMFPCA methods.}
  \label{mare 2}
\end{figure}

\subsection{Sparse Case}
In this simulation setup, we adopt the following functional data generating model:

\[
{\pmb{X}_i(\pmb{t})} =
\begin{cases} 
    \sum_{m=1}^{3} {\rho}^{(1)}_{i,m} \pmb\psi_m (\pmb{t}) + \pmb{\epsilon}_{1}(\pmb{t}) & \text{if } i \in S_1 \\
    \sum_{m=1}^{3} {\rho}^{(2)}_{i,m} \pmb\psi_m (\pmb{t}) + \pmb{\epsilon}_{2}(\pmb{t})& \text{if } i \in S_2 \\
    \sum_{m=1}^{3} {\rho}^{(3)}_{i,m} \pmb\psi_m (\pmb{t}) + \pmb{\epsilon}_{3}(\pmb{t}) & \text{if } i \in S_3, \\
\end{cases}
\]

where $\rho^{(k)}_{i,m}$ are i.i.d. normal with a mean of 0 and variances $\lambda_m^{(k)}$. The sets $S_k$ $(k = 1, 2, 3)$ represent index sets with lengths $N_k$, where $N_k$ denotes the number of observations in each set.

To demonstrate the superior performance of our approach, which effectively penalizes coefficients that should be zero, thereby enhancing the extraction of more informative features, we assume the following structure for \(\lambda_m^{(k)}\):

\[
\begin{bmatrix}
\lambda_1^{(1)} & \lambda_2^{(1)} & \lambda_3^{(1)} \\
\lambda_1^{(2)} & \lambda_2^{(2)} & \lambda_3^{(2)} \\
\lambda_1^{(3)} & \lambda_2^{(3)} & \lambda_3^{(3)}
\end{bmatrix}
=
\begin{bmatrix}
0 & 0.5 & 0.9 \\
0.9 & 0 & 0.5 \\
0.5 & 0.9 & 0
\end{bmatrix}.
\]

In this setup, when \(\lambda_m^{(k)} = 0\), the corresponding \(\rho^{(k)}_{i,m} = 0\), leading to the exclusion of the pattern \(\pmb\psi_m (\pmb{t})\) from the observation \(\pmb{X}_i(\pmb{t})\) for \(i \in S_k\). By adopting this structure for \(\lambda_m^{(k)}\), we generate observations with varying degrees of sparsity in the PC scores corresponding to each \(\pmb\psi_m (\pmb{t})\). The noise \(\pmb{\epsilon}_{k}(\pmb{t})\) follows an i.i.d. normal distribution with a mean of \(\pmb{0}\) and a covariance matrix \(\pmb\Sigma_k\):

$$\begin{bmatrix}
\sigma_k^2 & \rho\sigma_k^2  \\
\rho\sigma_k^2 & \sigma_k^2
\end{bmatrix},$$ 
where $\sigma_k^2$ represents the variances of the measurement noise for the first and second variables, and $\rho$ represents the correlation between these two variables. In this simulation setup we set $\rho=0.4$, $\sigma_1^2=2.5$, $\sigma_2^2=1.5$, and $\sigma_3^2=0.5$.

We consider three sample sizes: $N = 500$, $1000$, and $2000$. For each $N$, data is generated under three different scenarios, each defined by a specific ratio of $N_1 : N_2 : N_3$. In Scenario I, the ratio is $5\% : 60\% : 35\%$; in Scenario II, it is $15\% : 50\% : 35\%$; and in Scenario III, the ratio is $25\% : 40\% : 35\%$. Each scenario is evaluated using 500 simulations. In all three scenarios, $\pmb\psi_1 (\pmb{t})$ represents the leading PC, with $N_1$ out of $N$ PC scores set to zero; $\pmb\psi_2 (\pmb{t})$ corresponds to the second PC, with $N_2$ out of $N$ PC scores set to zero; and $\pmb\psi_3 (\pmb{t})$ represents the third PC, with $N_3$ out of $N$ PC scores set to zero. In MFPCA, the variance captured by each functional PC decreases sequentially, indicating that higher sparsity in the leading PC scores has a more pronounced negative effect on the estimation results. In our simulations, Scenarios I, II, and III correspond to low, medium, and high sparsity in the leading PC scores, respectively.

\begin{figure}[!t] 
  \centering
  \begin{subfigure}{0.5\textwidth}
    \centering
    \includegraphics[page = 1, width=\textwidth]{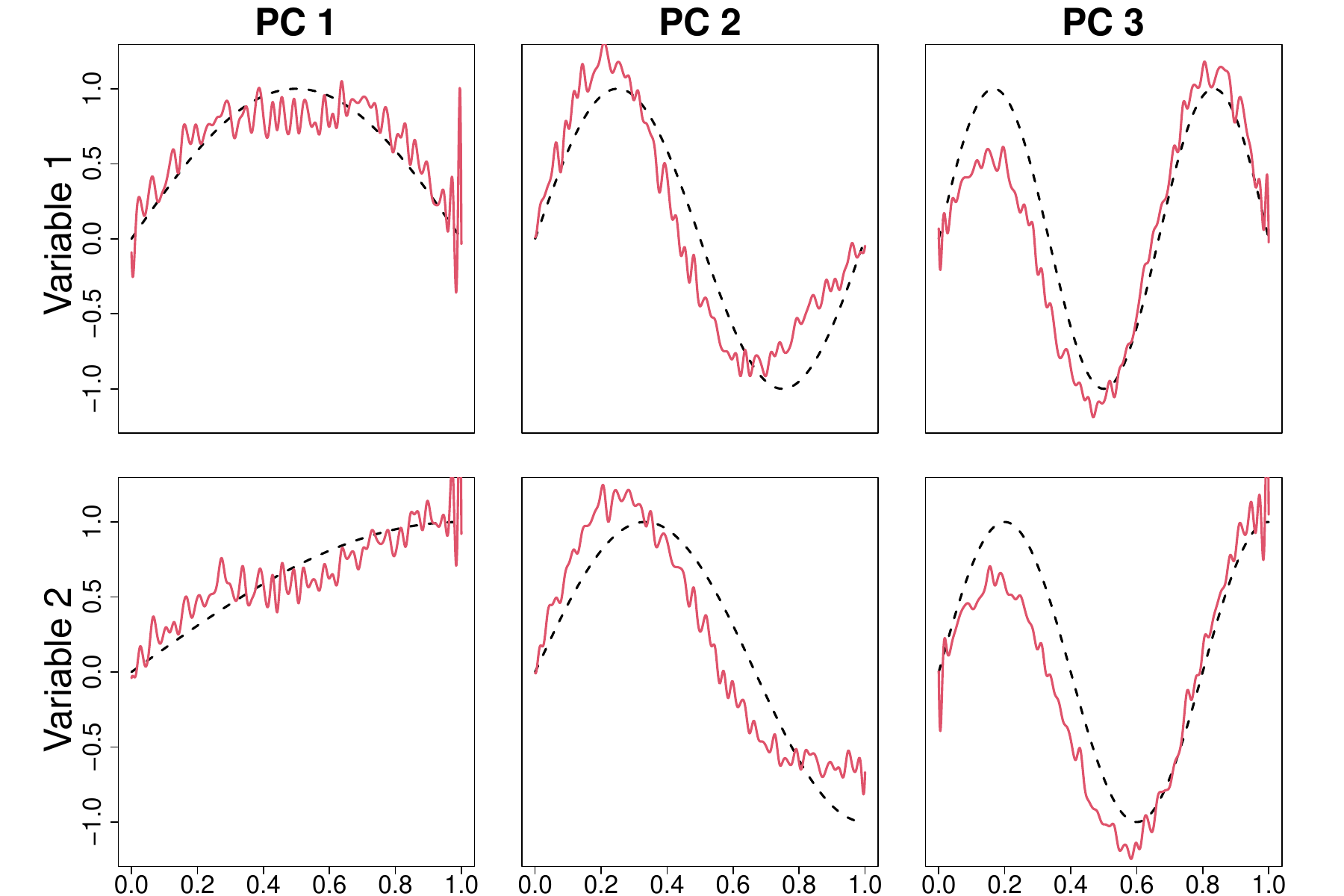}
    \caption{Non-regularized  MFPCA}
    \label{MFPCA}
  \end{subfigure}%
  \begin{subfigure}{0.5\textwidth}
    \centering
    \includegraphics[page = 2,width=\textwidth]{pc_example.pdf}
    \caption{Smooth MFPCA}
    \label{Smoothed MFPCA}
  \end{subfigure}
    \begin{subfigure}{0.5\textwidth}
    \centering
    \includegraphics[page = 3,width=\textwidth]{pc_example.pdf}
    \caption{Sparse MFPCA}
    \label{Sparse MFPCA}
  \end{subfigure}%
  \begin{subfigure}{0.5\textwidth}
    \centering
    \includegraphics[page = 4,width=\textwidth]{pc_example.pdf}
    \caption{Smooth and sparse MFPCA}
    \label{Smooth and Sparse MFPCA}
  \end{subfigure}
  \caption{Estimated functional principal components from a randomly selected simulation replicate under the high-sparsity setting (Scenario III). 
The estimated functional PCs are shown in red, while the true functional PCs are shown as black dashed curves. 
Panels compare non-regularized MFPCA, smooth MFPCA, sparse MFPCA, and the proposed smooth-and-sparse ReMFPCA approach.}
  \label{sparse FPC example}
\end{figure}

We evaluate the performance of our smooth-and-sparse ReMFPCA approach against three alternative methods: conventional MFPCA, sequential smooth MFPCA, and sparse MFPCA. Although the data-generating mechanism in this sparse simulation setting has a mixture-like structure through the subject groups $S_1,S_2,S_3$, the objective of this simulation is different from model-based clustering of functional trajectories. Our focus is on regularized dimension reduction, specifically on recovering the underlying functional PCs and identifying subject-level sparsity in the associated PC scores. Figure \ref{sparse FPC example} illustrates a randomly selected simulation with high sparsity (Scenario III) where \( N = 2000 \). This figure demonstrates that the smooth-and-sparse ReMFPCA approach successfully captures the optimal functional PCs. The key to this success lies in the method's ability to effectively penalize irrelevant \(\rho_{i,m}^{k}\), reducing them to zero (see the results of the estimated sparse PC scores \(\hat\rho_{i,m}^{k}\) in Figure \ref{pc score}).

The accuracy of the estimated $\hat{\pmb\psi}_m$ across all simulation setups was evaluated using the error measure: $Err(\hat{\pmb\psi}_m) = \Vert\hat{\pmb\psi}_m-\pmb\psi_m\Vert_{\mathbb{H}}$. The performance of each method is illustrated in Figure \ref{sparse function error}. Additionally, the performance of an alternative measure, $Err(\hat\rho_{i,m}^{k}) = \frac{1}{N}\sum_{i=1}^{N}|\hat\rho_{i,m}^{k}-\rho_{i,m}^{k}|$, is presented in Figure \ref{sparse score error}.

As depicted in Figure \ref{sparse function error}, the differences between the approaches are less pronounced under low sparsity conditions. However, in high sparsity scenarios, where the sparsity has a more substantial impact on the estimated PC scores, our method shows significant improvement over non-sparse penalty methods. Furthermore, our algorithm effectively assigns varying sparsity parameters to each functional PC, leading to more accurate estimation of the functional PCs. 

\begin{figure}[!t]
    \centering
    \includegraphics[width=1\textwidth]{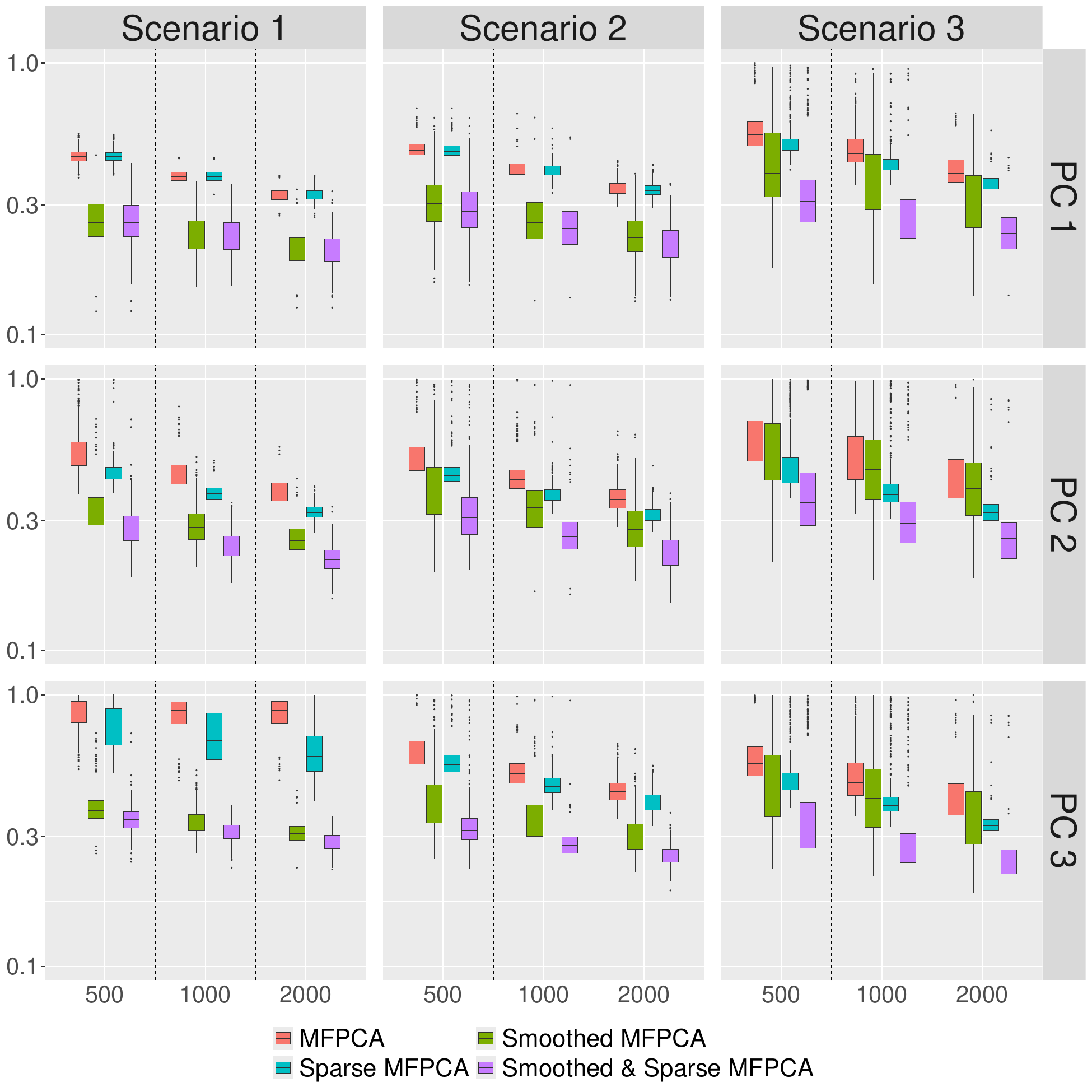}
    \caption{Estimation error of the functional principal components in the sparse simulation setting. 
The error is measured by 
$Err(\hat{\pmb\psi}_m)=\Vert \hat{\pmb\psi}_m-\pmb\psi_m\Vert_{\mathbb{H}}$, 
with comparisons across different sparsity levels, sample sizes, and competing methods.}
    \label{sparse function error}
\end{figure}

\section{Real Data Analysis}
\label{Real data analysis}
In this section, we analyze a dataset comprising measurements of user acceleration and pitch attitude from 24 individuals performing four distinct activities: jogging, walking, sitting, and standing \cite{Malekzadeh}. Each activity was conducted for approximately 2 to 3 minutes. The irregularly observed discrete data were transformed into functional objects using 200 B-spline basis functions. Additionally, to ensure that both variables contribute equally to the variation in the analysis, we followed the rescaling method proposed by \cite{happ2018multivariate}, applying weights $w_j$:
$$w_j = \left(\int_{\mathcal{T}_i} \widehat{Var}(X_j(t))dt\right)^{-1},\qquad j = 1,2.$$
Using these weights, the integrated variance equals 1 for the rescaled variable $\tilde{X}_j(t) = w_j^{1/2}X_j(t).$

To obtain informative functional PCs, we employ the proposed sequential ReMFPCA approach, which integrates both roughness and sparsity penalties. By extracting each functional PC sequentially, we are able to tailor the tuning parameters specifically for each component. We investigate a series of candidate values for each $\alpha_j$ in the smoothing parameter vector $\pmb{\alpha} = (\alpha_1, \alpha_2)$, where $\alpha_j$ is set as $2^i$, with $i$ representing 10 equally spaced values within the range $[-35, 5]$. Additionally, we consider a spectrum of candidate values for the degrees of sparsity, generating these as integers from 1 to 96, where 96 corresponds to the total number of observations. This thorough approach allows us to explore all potential levels of sparsity and smoothness, ensuring the identification of the optimal parameters for our analysis.

The optimal smoothing parameters were identified as follows: $\pmb{\alpha}_{\text{seq}}^{1} = (6.75 \times 10^{-2}, 0.31 \times 10^{-2})$, and $\pmb{\alpha}_{\text{seq}}^{2} = \pmb{\alpha}_{\text{seq}}^{3} = (6.75 \times 10^{-2}, 1.469)$. The optimal level of sparsity for the first functional PC was determined to be 25, resulting in some observations having a first PC score of 0. For the remaining functional PCs, the optimal sparsity level is 1, indicating that the sparsity penalty has a minimal effect on further refining these PCs. The impact of adding and subtracting multiples of each sequentially regularized functional PC on the mean curve is illustrated in Figure \ref{FPC plot:plot5}.

For comparison, we evaluate our smooth-and-sparse MFPCA approach against several alternatives, including the non-regularized MFPCA, as well as joint and sequential smoothing MFPCA methods that utilize only a roughness penalty. The sequential smoothing MFPCA yields the same smoothing parameters as our smooth-and-sparse MFPCA, while the joint smoothing MFPCA produces the parameters $\pmb\alpha_{joint}^{smooth} = (6.75 \times 10^{-2}, 6.75 \times 10^{-2})$. These are illustrated in Figure \ref{FPC plot:plot1}.

\begin{figure}[!t]
    \centering
    \includegraphics[width=1\textwidth]{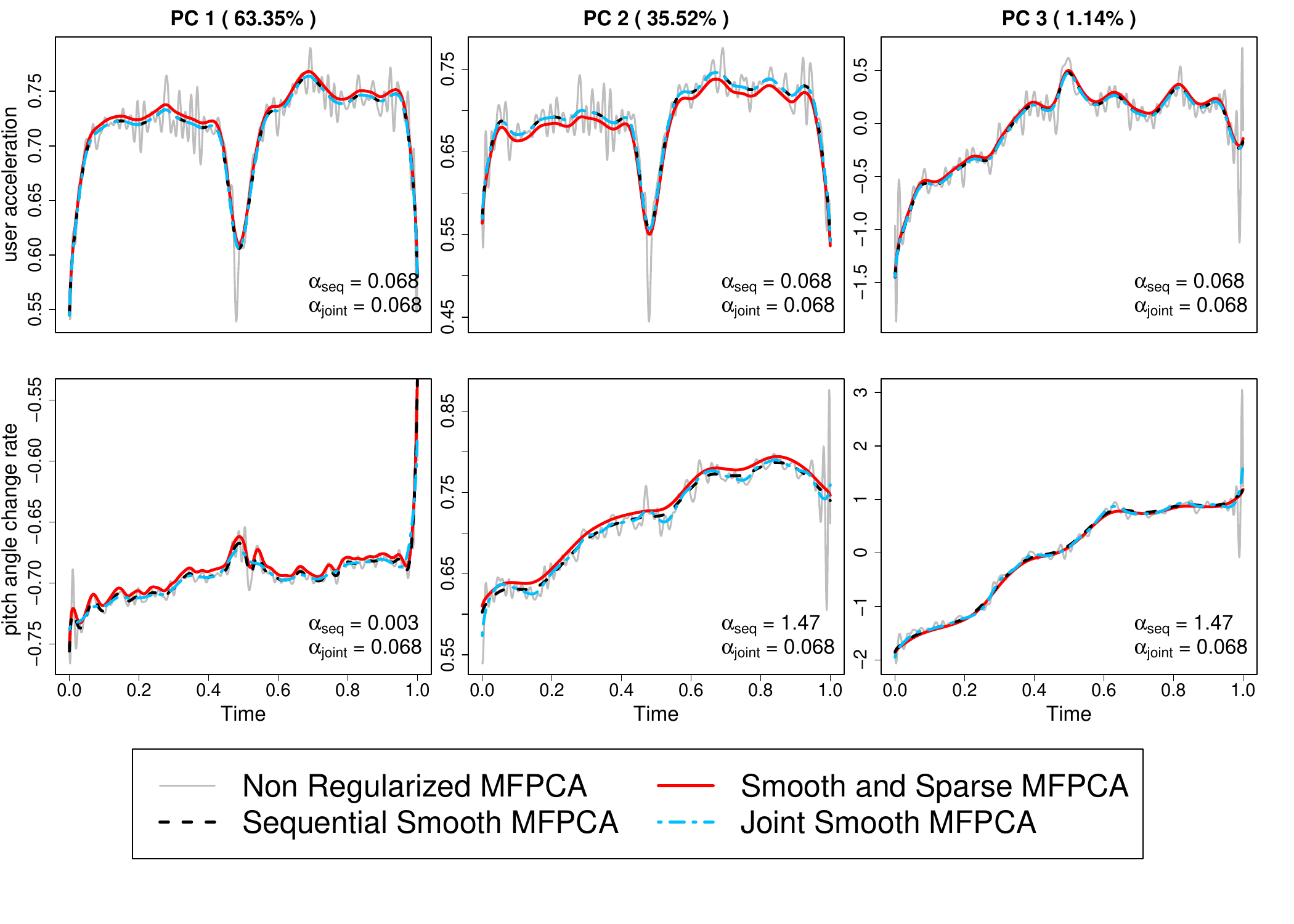}
    \caption{First three estimated functional principal components for the MotionSense data. 
The curves compare non-regularized MFPCA in red, sequential ReMFPCA in black, and joint ReMFPCA in blue.}
    \label{FPC plot:plot1}
  \end{figure}

\begin{figure}[!b]
    \centering
    \includegraphics[width=1\textwidth]{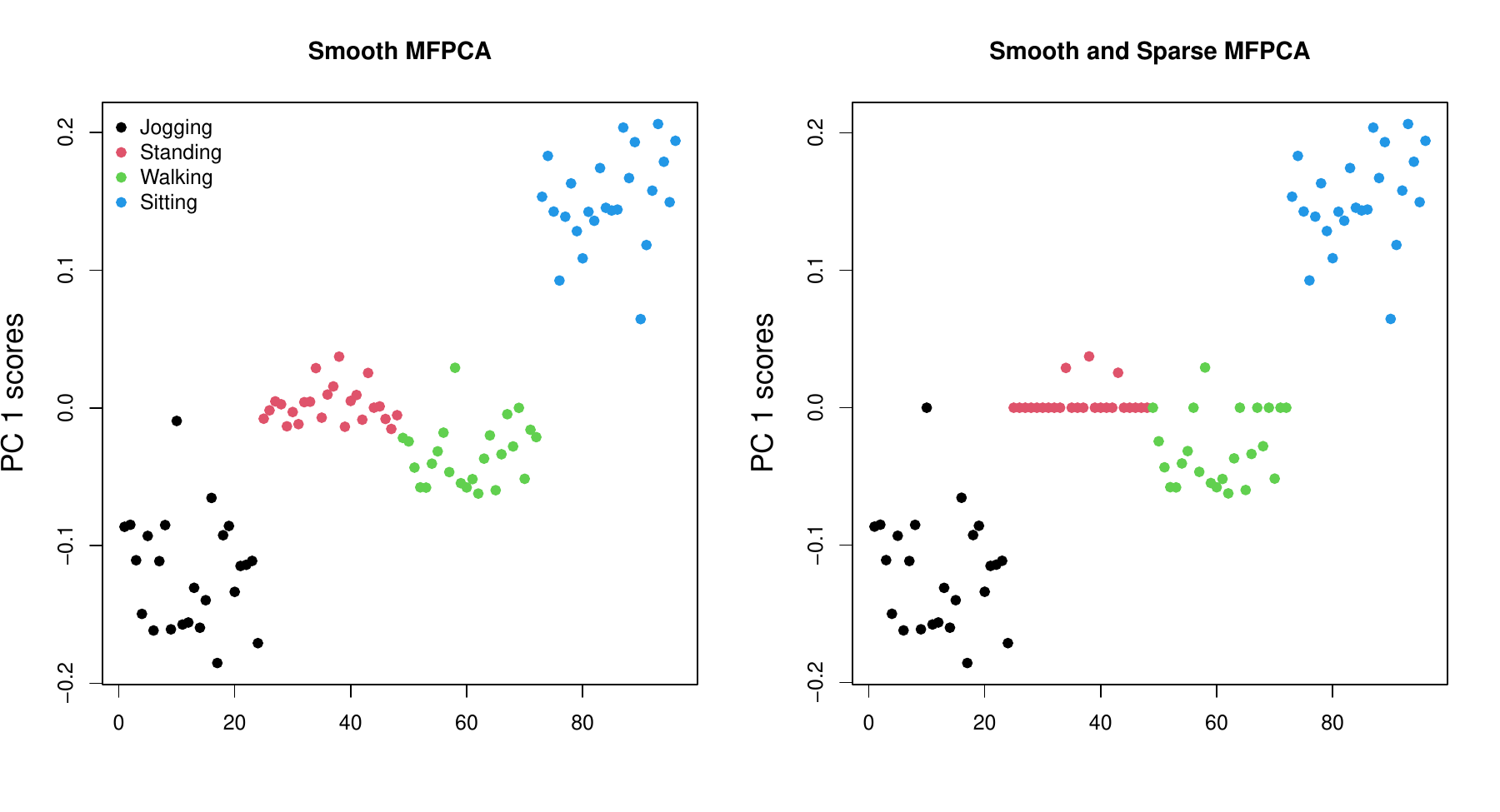}
    \caption{Scatter plot of the first estimated PC scores from the sequential smooth-and-sparse ReMFPCA analysis of the MotionSense data. 
The sparsity penalty shrinks a subset of activity-specific scores toward zero, improving the interpretability of the first component.}
    \label{sparse MFPCA score plot}
\end{figure}

Figure \ref{sparse MFPCA score plot} reveals that the PC1 scores for the standing activity are effectively penalized to near zero. This indicates that the sparsity penalty efficiently filters out most of the information from the standing activity when estimating the first PC, reinforcing the notion that excluding this information does not compromise the analysis. This observation is consistent with the results in Figures \ref{pc1 cluster svd} and \ref{pc2 cluster svd}, which suggest that the standing activity is primarily captured by the second PC and does not significantly influence the first.

\begin{figure}[!t] 
  \centering
  \begin{subfigure}{0.33\textwidth}
    \centering
    \includegraphics[width=\textwidth]{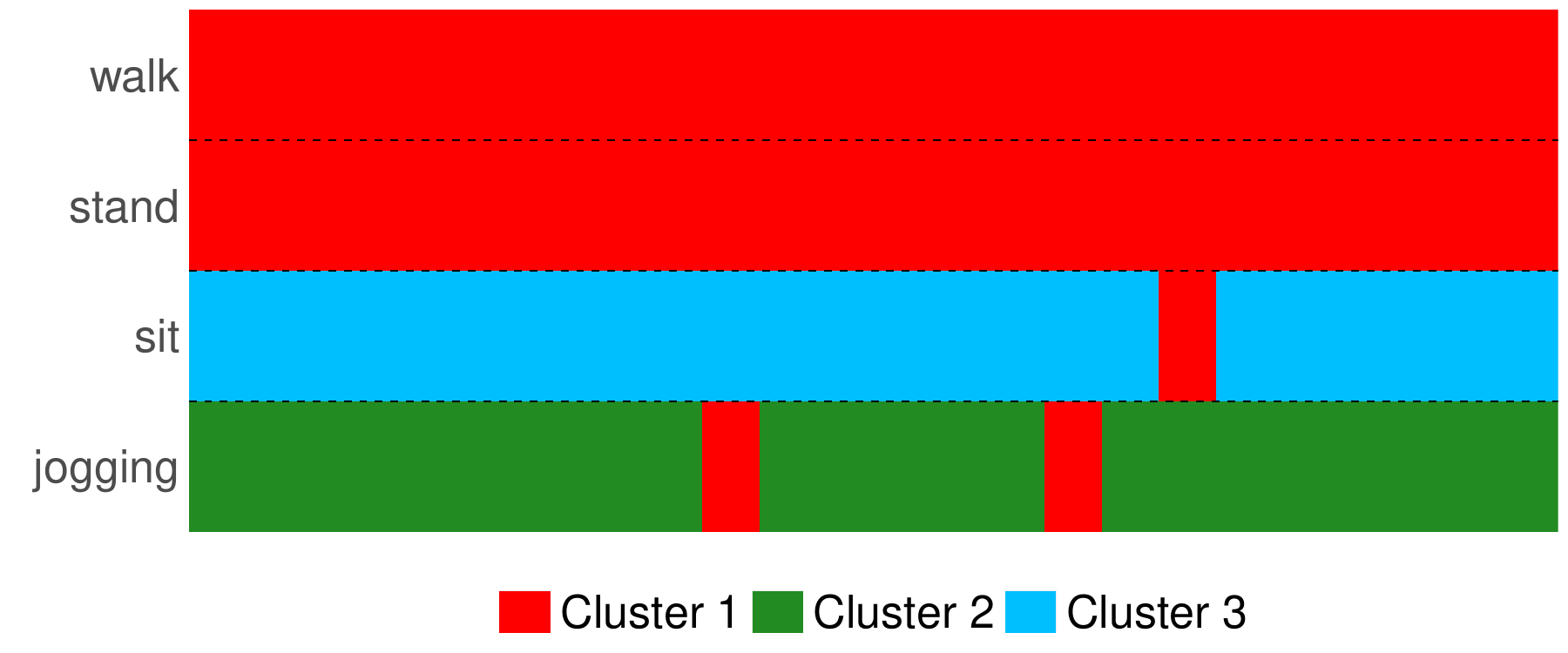}
    \caption{Clustering using PC1}
    \label{pc1 cluster svd}
  \end{subfigure}%
  \begin{subfigure}{0.33\textwidth}
    \centering
    \includegraphics[width=\textwidth]{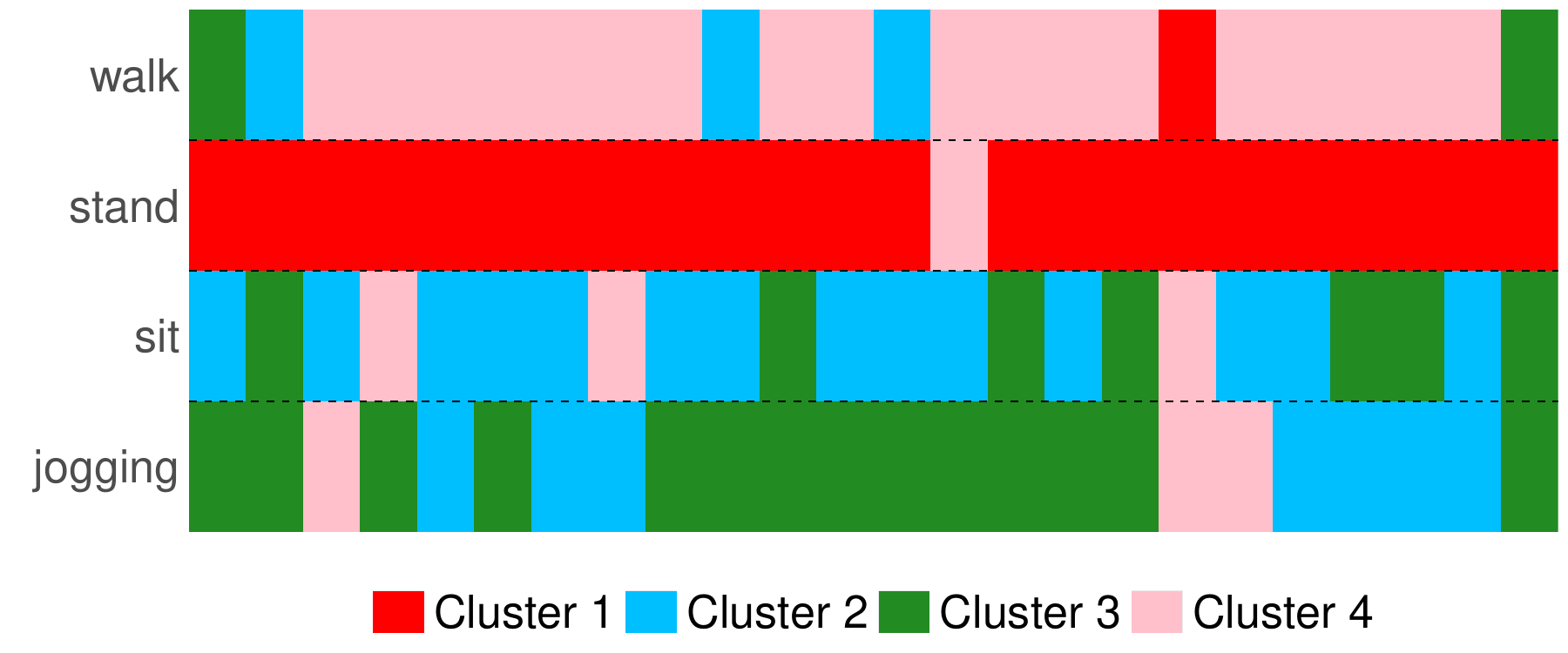}
    \caption{Clustering using PC2}
    \label{pc2 cluster svd}
  \end{subfigure}
    \begin{subfigure}{0.33\textwidth}
    \centering
    \includegraphics[width=\textwidth]{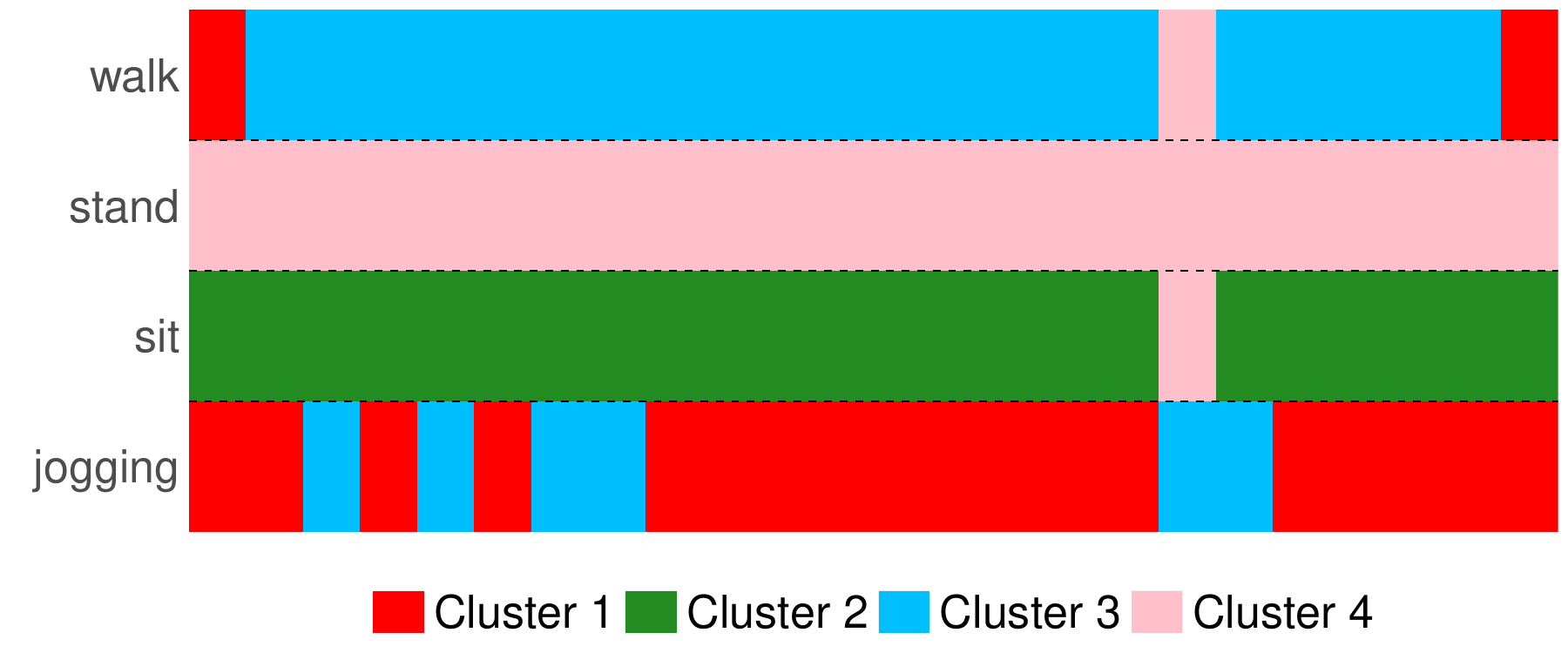}
    \caption{Clustering using PC1 and PC2}
    \label{pc3 cluster svd}
  \end{subfigure}
  \caption{Clustering results based on the estimated functional PC scores from the proposed ReMFPCA analysis of the MotionSense data. 
The three panels show clustering using the first PC, the second PC, and the first two PCs jointly, respectively. 
The combined PC1--PC2 representation provides improved separation of the activity-specific functional patterns.}
   \label{fsvd cluster plot}
\end{figure}

Figure \ref{fsvd cluster plot} illustrates the clustering outcomes derived from the scores of the first, second, and a combination of both functional PCs. For these clustering tasks, the silhouette criterion identifies the optimal number of clusters as $k = 3, 4, 4$ for the first, second, and combined PCs, respectively (details on the silhouette criterion can be found in Figure \ref{Silhouette plot}). The clustering based on the first PC scores effectively captures the patterns in sitting and jogging activities, while the second PC scores successfully reflect the pattern of standing.

The combination of the first two functional PCs clearly demonstrates their ability to differentiate the distinct patterns associated with each activity. This clustering performance is quantified using the Normalized Mutual Information (NMI) and the Adjusted Rand Index (ARI), yielding scores of NMI = 0.79 and ARI = 0.78, respectively. The interpretability of these results is further supported by Figure \ref{cluster}, which compares the clustering outcomes with the true activity labels.

\section{Discussion} \label{Discussion}
This paper presents a generalized functional SVD framework to investigate ReMFPCA for $p$-dimensional functional data across multiple domains, where the case of $p=1$ reduces to the ReFPCA problem. While our approach to smoothing functional PCs is inspired by the methods introduced in \cite{huang2008functional}, it extends significantly beyond a univariate-to-multivariate adaptation, offering several key advantages. First, our implementation can accommodate certain irregularly or unequally spaced functional data within the adopted basis-expansion framework. However, because the basis representation depends on the relationship between the measurement points and the knot locations, this capability may be limited when the observations are very sparse. Thus, unlike approaches specifically developed for sparse functional designs, such as \cite{PACE}, the proposed method should not be interpreted as directly addressing very sparse functional data. Additionally, our approach is grounded in a Hilbert space framework, a foundational aspect in functional data analysis that was not explicitly established in \cite{huang2008functional}. This provides a sample-level operator-based formulation for the proposed functional SVD technique, rather than a full population-level infinite-dimensional theory. Furthermore, our Generalized Cross-Validation (GCV) criterion is rigorously validated for selecting smoothing parameters not only for individual PCs but also for determining common smoothing parameters across multiple PCs, offering greater versatility and applicability to methods such as those developed by \cite{kayano2009functional} and \cite{ReMFPCA}.

Our approach also differs fundamentally from \cite{huang2009}, which applied simultaneous roughness penalties to univariate two-way functional data. Instead, we target multivariate functional data across distinct domains in $\mathbb{R}$, incorporating a \textit{dual penalization} strategy that applies roughness penalties to the functional PCs and sparsity penalties to the associated PC scores. This dual penalization addresses two important challenges: roughness in the estimated functions and the presence of irrelevant subject-specific variations. The latter is particularly important, since without sparsity constraints, PC scores remain nonzero even for subjects unrelated to a given functional pattern, thereby introducing noise and reducing interpretability. This issue parallels challenges seen in sparse PCA \citep{spca_zou,spca_shen}, where nonzero loadings complicate interpretation. By shrinking irrelevant PC scores toward zero, our method extracts components that more clearly represent shared functional patterns, improving interpretability without compromising variance explanation.

Compared with \cite{spfca}, which imposed both roughness and sparsity penalties on functional PCs through an indirect iterative power algorithm, the present framework generalizes these ideas to multivariate settings and provides improved computational efficiency through regression-based cross-validation. Within this functional SVD framework, we propose two power algorithms for estimating functional PCs, sequentially and jointly, while maintaining computational efficiency in tuning parameter selection. The sequential power algorithm allows for the extraction of functional PCs one at a time, each with its own tuning parameters, resulting in more tailored PCs. The joint power algorithm is more computationally efficient, as it estimates multiple functional PCs simultaneously using common tuning parameters.

An interesting avenue for future research is the extension of this framework to partially observed functional data. In many practical applications, functions may be observed only on a subset of their domain, often synchronously across subjects. Exploring such an extension would further broaden the applicability of the method to a wider range of real-world problems.

\bigskip
\begin{center}
{\large\bf SUPPLEMENTAL MATERIALS}
\end{center}

\begin{description}

\item[Appendix:] The proofs of all the theorems and lemmas are provided in the Appendix. (.pdf)

\item[R-package for ReMFPCA routine:] We implemented our smooth-and-sparse technique by extending the R-package \pkg{ReMFPCA}  with the necessary functionalities. (available in CRAN)

\item[MotionSense Data example:] The `code.R'  contains all codes and datasets used as the real data example in this article.

\end{description}

\bigskip
\begin{center}
{\large\bf Data Availability}
\end{center}

\noindent
The experimental data supporting the findings of this study are publicly available. 
The raw MotionSense dataset can be accessed from the MotionSense repository: 
\url{https://github.com/mmalekzadeh/motion-sense}. 
A cleaned version of this dataset is also provided within the R package \texttt{ReMFPCA}, which is available on CRAN.

\bibliographystyle{apalike}
\bibliography{Mybib}

\clearpage
\pagenumbering{arabic}
\renewcommand*{\thepage}{A\arabic{page}}
\appendix
\renewcommand{\thesection}{\Alph{section}}
\setcounter{equation}{0}
\renewcommand\theequation{\Alph{section}.\arabic{equation}}
\setcounter{figure}{0}
\renewcommand\thefigure{\Alph{section}.\arabic{figure}}
\include{appendix}

\end{document}

%% file: appendix.tex
{\center \section*{APPENDIX}}
\section{Proofs}\label{proofs}
\begin{proof}[Proof of Proposition \ref{prop 2.2}]
By the definition of the data operator, we have $\mathbfcal{X}\circ\mathbfcal{X}^*: \mathbb{H} \rightarrow \mathbb{H}$. For any ${\pmb x} \in \mathbb{H}$, the operator acts as follows:
\begin{align*}
    \mathbfcal{X}\circ\mathbfcal{X}^*({\pmb x}) &= \mathbfcal{X}\left( 
    \begin{pmatrix} \langle {\pmb x}_1, {\pmb x} \rangle_{\mathbb{H}}, \langle {\pmb x}_2, {\pmb x} \rangle_{\mathbb{H}}, \ldots, \langle {\pmb x}_n, {\pmb x} \rangle_{\mathbb{H}} \end{pmatrix}^\top \right) \\
    &= \sum_{i=1}^n \langle {\pmb x}_i, {\pmb x} \rangle_{\mathbb{H}} {\pmb x}_i \quad = \quad \sum_{i=1}^n {\pmb x}_i \otimes {\pmb x}_i ({\pmb x}).
\end{align*}
\end{proof}

 \begin{proof}[Proof of Theorem \ref{th-fsvd}]
See proofs of Theorem 1 in \citep{Haghbin2019fssa}.
	\end{proof}  

\begin{lemma}
\label{lem1}
For $\pmb{u} \in \mathbb{R}^n$ and $\pmb{\psi} \in \mathbb{H}$, we have $\pmb{u} \otimes \pmb{\psi} \in \mathbb{F}^{p \times n}$, and the following holds:
\begin{equation}
\label{e00}
\Vert \pmb{u} \otimes \pmb{\psi} \Vert_\mathbb{F} = \Vert \pmb{u} \Vert_{\mathbb{R}^n} \Vert \pmb{\psi} \Vert_{\mathbb{H}}.
\end{equation}
\end{lemma}
\begin{proof}[Proof]
Let $\pmb{u} = (u_1, \ldots, u_n)^{\top}$. Then, for any vector $\pmb{a} = (a_1, \ldots, a_n)^{\top} \in \mathbb{R}^n$, we have
\begin{equation*}
    \pmb{u} \otimes \pmb{\psi} (\pmb{a}) = \langle \pmb{u}, \pmb{a} \rangle_{\mathbb{R}^n} \pmb{\psi} = \sum_{i=1}^n a_i u_i \pmb{\psi}.
\end{equation*}
Hence, $\pmb{u} \otimes \pmb{\psi} \in \mathbb{F}^{n \times p}$ is specified by $[u_i \pmb{\psi}]_{i=1}^{n}$. This gives
\begin{equation*}
    \Vert \pmb{u} \otimes \pmb{\psi} \Vert_\mathbb{F} = \sqrt{\sum_{i=1}^n \Vert u_i \pmb{\psi} \Vert^2_{\mathbb{H}}} = \sqrt{\Vert \pmb{\psi} \Vert^2_{\mathbb{H}} \sum_{i=1}^n u_i^2} = \Vert \pmb{u} \Vert_{\mathbb{R}^n} \Vert \pmb{\psi} \Vert_{\mathbb{H}}.
\end{equation*}
\end{proof}

\begin{proof}[Proof of Theorem \ref{col1}]

\begin{enumerate}
    \item [i)] 
    \begin{align*}
        \mathbfcal{X} \circ \mathbfcal{X}^* (\pmb{\psi}_h) &= \mathbfcal{X} \left( \sum_{\ell=1}^L \sqrt{\lambda_\ell} \langle \pmb{\psi}_\ell, \pmb{\psi}_h \rangle_{\mathbb{H}} \pmb{u}_i \right) 
        = \mathbfcal{X} (\sqrt{\lambda_h} \pmb{u}_h) 
        = \sqrt{\lambda_h} \mathbfcal{X}(\pmb{u}_h) \\
        &= \sqrt{\lambda_h} \sum_{\ell=1}^L \sqrt{\lambda_\ell} \langle \pmb{u}_\ell, \pmb{u}_h \rangle_{\mathbb{R}^n} \pmb{\psi}_\ell
        = \lambda_h \pmb{\psi}_h.
    \end{align*}
    Therefore, 
    \begin{equation*}
        {\widehat{\mathbfcal{C}}} (\pmb{\psi}_\ell) = \frac{1}{n-1} \mathbfcal{X} \circ \mathbfcal{X}^* (\pmb{\psi}_\ell) = \frac{\lambda_\ell}{n-1} \pmb{\psi}_\ell.
    \end{equation*}

    \item [ii, iii)] See (3.1) and (3.2) of \cite{ReMFPCA}.

    \item [iv)] Using Lemma \ref{lem1}, we have:
    \begin{equation} \label{e9}
        \Vert \mathbfcal{X} \Vert_\mathbb{F}^2 = \Vert \sum_{\ell=1}^L \sqrt{\lambda_\ell} \pmb{u}_\ell \otimes \pmb{\psi}_\ell \Vert_\mathbb{F}^2 
        = \sum_{\ell=1}^L \lambda_\ell \Vert \pmb{u}_\ell \otimes \pmb{\psi}_\ell \Vert_\mathbb{F}^2 
        = \sum_{\ell=1}^L \lambda_\ell.
    \end{equation}
\end{enumerate}
\end{proof}
 
    The following lemma is essential to finalize the proof of Theorem \ref{Low-rank approximation}.
\begin{lemma}\label{e10}
Let the orthonormal elements $\pmb{\varphi}_\ell,\ \ell=1,\ldots,r,$ be fixed. Then, the optimal $\pmb{w}_i$ that minimizes 
\[
\Vert \mathbfcal{X} - \sum_{\ell=1}^r \ \pmb{w}_\ell \otimes \pmb{\varphi}_\ell \Vert_\mathbb{F}^2,
\]
has the form $\pmb{w}_\ell = \mathbfcal{X}^*(\pmb{\varphi}_\ell)$.
\end{lemma}

\begin{proof}

Define $\mathbb{M}_r$ as the space of all elements expressible in the form
$\sum_{\ell=1}^r \pmb{w}_\ell \otimes \pmb{\varphi}_\ell$
where $\pmb{w}_\ell \in \mathbb{R}^n$ and $\pmb{\varphi}_\ell \in \mathbb{H}$, with the functions $\pmb{\varphi}_\ell$ forming an orthonormal set. Now let $\pmb{w}_\ell=(w_1^{(\ell)},\ldots,w_n^{(\ell)})^\top$, then $\pmb{w}_\ell \otimes \pmb{\varphi}_\ell$ belongs to $\mathbb{F}^{p\times n}$ and is specified by $[w_i^{(\ell)}\pmb{\varphi}_\ell]_{i=1}^n$. Therefore,
\begin{align}\label{e7}
\langle \mathbfcal{X}, \sum_{\ell=1}^r \pmb{w}_\ell \otimes \pmb{\varphi}_\ell \rangle_\mathbb{F}
&= \sum_{\ell=1}^r \langle \mathbfcal{X}, \pmb{w}_\ell \otimes \pmb{\varphi}_\ell \rangle_\mathbb{F} \ = \ \sum_{\ell=1}^r \sum_{i=1}^n w_i^{(\ell)} \langle \pmb{x}_i, \pmb{\varphi}_\ell \rangle_{\mathbb{H}} \notag \\
&= \sum_{\ell=1}^r \langle \mathbfcal{X}^*(\pmb{\varphi}_\ell), \pmb{w}_\ell \rangle_{\mathbb{R}^n}.
\end{align}
Also,
\begin{align}\label{e8}
\langle \sum_{\ell=1}^r \mathbfcal{X}^*(\pmb{\varphi}_\ell) \otimes \pmb{\varphi}_\ell, \sum_{t=1}^r \pmb{w}_t \otimes \pmb{\varphi}_t \rangle_\mathbb{F}
&= \sum_{\ell=1}^r \sum_{t=1}^r \langle \mathbfcal{X}^*(\pmb{\varphi}_\ell) \otimes \pmb{\varphi}_\ell, \pmb{w}_t \otimes \pmb{\varphi}_t \rangle_\mathbb{F} \notag \\
&= \sum_{\ell=1}^r \sum_{t=1}^r \langle \mathbfcal{X}^*(\pmb{\varphi}_\ell), \pmb{w}_t \rangle_{\mathbb{R}^n} \langle \pmb{\varphi}_\ell, \pmb{\varphi}_t \rangle_{\mathbb{H}} \notag \\
&= \sum_{\ell=1}^r \langle \mathbfcal{X}^*(\pmb{\varphi}_\ell), \pmb{w}_\ell \rangle_{\mathbb{R}^n}.
\end{align}
Using \eqref{e7} and \eqref{e8}, we obtain
\begin{align*}
    \biggl\langle \mathbfcal{X} - \sum_{\ell=1}^r \mathbfcal{X}^*(\pmb{\varphi}_\ell) \otimes \pmb{\varphi}_\ell, \sum_{t=1}^r \pmb{w}_t \otimes \pmb{\varphi}_t \biggr\rangle_\mathbb{F} &= 
    \Bigl\langle \mathbfcal{X}, \sum_{t=1}^r \pmb{w}_t \otimes \pmb{\varphi}_t \Bigr\rangle_\mathbb{F} \\
    &- \biggl\langle \sum_{\ell=1}^r \mathbfcal{X}^*(\pmb{\varphi}_\ell) \otimes \pmb{\varphi}_\ell, \sum_{t=1}^r \pmb{w}_t \otimes \pmb{\varphi}_t \biggr\rangle_\mathbb{F} = 0.
\end{align*}
Thus, the operator $\sum_{\ell=1}^r \mathbfcal{X}^*(\pmb{\varphi}_\ell) \otimes \pmb{\varphi}_\ell$ is the orthogonal projection of the operator $\mathbfcal{X}$ onto the subspace of $\mathbb{M}_r$. This completes the proof.
\end{proof}
			
			\begin{proof}[Proof of the Theorem \ref{Low-rank approximation}]
				According to Lemma \ref{e10}, for given orthonormal vectors $\pmb{\varphi}_\ell$, the optimal vectors $\pmb{w}_\ell$ are expressed as $\pmb{w}_\ell = \mathbfcal{X}^*(\pmb{\varphi}_\ell)$. Hence, the problem is reduced to determining the optimal $\pmb{\varphi}_\ell$. Using \eqref{e7} and \eqref{e8}, we obtain:
				\begin{align*}
					{\Vert \mathbfcal{X} - \sum_{\ell=1}^r \ \mathbfcal{X}^*(\pmb{\varphi}_\ell) \otimes \pmb{\varphi}_\ell \Vert_\mathbb{F}^2}&= 
					{\Vert \mathbfcal{X} \Vert_\mathbb{F}^2} -2 \langle\mathbfcal{X}  , \sum_{\ell=1}^r \ \mathbfcal{X}^*(\pmb{\varphi}_\ell)\otimes \pmb{\varphi}_\ell\rangle_\mathbb{F}
					+ {\Vert  \sum_{\ell=1}^r \ \mathbfcal{X}^*(\pmb{\varphi}_\ell) \otimes \pmb{\varphi}_\ell \Vert_\mathbb{F}^2}\\
					&= {\Vert \mathbfcal{X} \Vert_\mathbb{F}^2} -\sum_{\ell=1}^r \langle \mathbfcal{X}^*(\pmb{\varphi}_\ell) , \mathbfcal{X}^*(\pmb{\varphi}_\ell)\rangle_{\mathbb{R}^n}\\
					&= {\Vert \mathbfcal{X} \Vert_\mathbb{F}^2} -\sum_{\ell=1}^r \langle \mathbfcal{X}\circ\mathbfcal{X}^*(\pmb{\varphi}_\ell) , \pmb{\varphi}_\ell\rangle_{\mathbb{H}}.
				\end{align*}
				The goal is to find orthonormal vectors $\pmb{\varphi}_\ell$ that minimize the right-hand side of this expression. This problem is well-established, as noted in Theorem \eqref{col1}. The vectors $\pmb{\varphi}_\ell$ can be chosen as the leading $r$ eigenfunctions of the operator $\mathbfcal{X} \circ \mathbfcal{X}^*$. Thus, we have $\pmb{\varphi}_\ell = \pmb{\psi}_\ell$. The proof is complete.
			\end{proof}	

			\begin{proof}[Proof of the Corollary \ref{Corollary SVD ED}] 
			In \cite{ReMFPCA}, the second-derivative operator $\mathcal{R}$ is introduced, which leads to the definition of the fourth-derivative operator $\mathcal{Q}:=\mathcal{R}^*\mathcal{R}$. Based on these definitions, the operator $S_{\alpha_j}$ is defined as:
		\begin{equation*}
			\mathcal{S}_{\alpha_j} = \left( \mathcal{I}+\alpha_j \mathcal{Q}\right)^{-1/2},\quad \alpha_j>0,\ j=1,\ldots, p.
		\end{equation*}
		Using these definitions, \cite{ReMFPCA} further defines the operator $\mathbfcal{S}_{\pmb{\alpha}}: \mathbb{H} \rightarrow \mathbb{H}$, given by:
	\begin{align*}
		\mathbfcal{S}_{\pmb{\alpha}}(\pmb f)(\pmb t):=
		\begin{bmatrix}
			\mathcal{S}_{\alpha_1}(f_1)(t_1)\\	
			\vdots\\
			\mathcal{S}_{\alpha_p}(f_p)(t_p)
		\end{bmatrix},
	\end{align*}
where $\pmb{t} = (t_1, \ldots, t_p) \in \mathbfcal{T}$ and $\pmb{f} = (f_1, \ldots, f_p) \in \mathbb{H}$.

It can be readily shown that $\tilde{\mathbfcal{X}}^* = \mathbfcal{X}^* \mathbfcal{S}_{\pmb{\alpha}}$. Thus, we have:
					\begin{equation*}
						\tilde{\mathbfcal{X}}^*\tilde{\pmb{\varphi}}=
						\mathbfcal{X}^*\mathbfcal{S}_{\pmb{\alpha}}\mathbfcal{S}_{\pmb{\alpha}}^{-1}\pmb{\varphi}=
						\mathbfcal{X}^*\pmb{\varphi}.
					\end{equation*}
					Now, considering \eqref{goal_new}, we have:
					\begin{align*}
						&=
						{\Vert \mathbfcal{X} \Vert_\mathbb{F}^2} - 2{\pmb{u}}^\top(\mathbfcal{X}^*\pmb{\varphi})+{\pmb{u}}^\top\pmb{u} \langle\pmb{\varphi},\pmb{\varphi}\rangle_{\pmb\alpha} \\
						&=
						{\Vert \mathbfcal{X} \Vert_\mathbb{F}^2}\pm {\Vert \tilde{\mathbfcal{X}} \Vert_\mathbb{F}^2} - 2{\pmb{u}}^\top(\tilde{\mathbfcal{X}}^*\tilde{\pmb{\varphi}})+{\pmb{u}}^\top\pmb{u} \langle\tilde{\pmb{\varphi}},\tilde{\pmb{\varphi}}\rangle_{\mathbb{H}}\notag\\
						&=
						{\Vert \mathbfcal{X} \Vert_\mathbb{F}^2}-{\Vert \tilde{\mathbfcal{X}} \Vert_\mathbb{F}^2}+
						\Vert \tilde{\mathbfcal{X}} - \pmb{u}\otimes\tilde{\pmb{\varphi}}\Vert_\mathbb{F}^2.
					\end{align*}
			\end{proof}
   
				\begin{proof}[Proof of the Corollary \ref{col4}]
					By applying Lemma 4.1 from \cite{ReMFPCA} and Corollary \ref{Corollary SVD ED}, this corollary follows directly.
				\end{proof}


	\begin{proof}[Proof of Lemma \ref{sparse penalty lemma}]
The equation \eqref{optimazation_problem} can be reformulated as:
\begin{align*}
\MoveEqLeft
 \sum_{i = 1}^{n} \langle \pmb{x}_i - u_i\pmb{\varphi} , \pmb{x}_i - u_i\pmb{\varphi} \rangle{_\mathbb{H}}  + {\pmb{u}}^\top\pmb{u}\sum_{j=1}^p \alpha_j\Vert { \mathcal{D}^2 \varphi_{j}} \Vert_{H_j}^2 + \sum_{i=1}^{n} p_{\gamma}(\lvert {{u_i}} \rvert)\\
 &= \sum_{i = 1}^{n} \langle \pmb{x}_i , \pmb{x}_i \rangle{_\mathbb{H}} - 2\langle \pmb{x}_i , u_i\pmb{\varphi}  \rangle{_\mathbb{H}} + \langle u_i\pmb{\varphi} , u_i\pmb{\varphi}  \rangle{_\mathbb{H}}  + {\pmb{u}}^\top\pmb{u}\sum_{j=1}^p \alpha_j\Vert { \mathcal{D}^2 \varphi_{j}} \Vert_{H_j}^2 + \sum_{i=1}^{n} p_{\gamma}(\lvert {{u_i}} \rvert) \\
 &= \sum_{i = 1}^{n} \langle \pmb{x}_i , \pmb{x}_i \rangle{_\mathbb{H}} - 2\langle \pmb{x}_i , u_i\pmb{\varphi}  \rangle{_\mathbb{H}} + {\pmb{u}}^\top\pmb{u} \langle \pmb{\varphi} , \pmb{\varphi}  \rangle{_{\pmb{\alpha}}} + \sum_{i=1}^{n} p_{\gamma}(\lvert {{u_i}} \rvert) \\
 &= \sum_{i = 1}^{n} \langle \pmb{x}_i , \pmb{x}_i \rangle{_\mathbb{H}} - 2\langle \pmb{x}_i , u_i\pmb{\varphi}  \rangle{_\mathbb{H}} + u^2_i + p_{\gamma}(\lvert {{u_i}} \rvert).
 \end{align*}
Disregarding the constant term $\sum_{i = 1}^{n} \langle \pmb{x}_i , \pmb{x}_i \rangle{_\mathbb{H}}$, the optimal $\hat{u}_i$ are those that minimize $u^2_i  - 2\langle \pmb{x}_i , u_i\pmb{\varphi}  \rangle{_\mathbb{H}} + p_{\gamma}(\lvert {{u_i}} \rvert)$, depending on the specific form of $p_{\gamma}(\cdot)$. The remainder of the proof is straightforward and thus omitted.
\end{proof}

\begin{proof}[Proof of Lemma \ref{CVlemma}]
	Recall that the criterion \eqref{pen ls eq} is associated with the penalized covariance matrix:
				\begin{equation*}
					\bar{\pmb{X}}^{\top}\bar{\pmb{X}}+\bar{\pmb{\Omega}}_{\pmb{\alpha}} = (\pmb{u}^{\top}\pmb{u})(\pmb{I}+\pmb{\Omega_\alpha}) = (\pmb{u}^{\top}\pmb{u})\tilde{\pmb{S}}^{-1}_{\pmb{\alpha}},
				\end{equation*}
                and the corresponding hat matrix:
				\begin{equation*}
					\pmb{H} = \bar{\pmb{X}}(\bar{\pmb{X}}^{\top}\bar{\pmb{X}}+\bar{\pmb{\Omega}}_{\pmb{\alpha}})^{-1}\bar{\pmb{X}}^{\top}
					={\frac{1}{\Vert\pmb{u}\Vert^2}\bar{\pmb{X}}\tilde{\pmb{S}}_{\pmb{\alpha}}}\bar{\pmb{X}}^{\top}.
				\end{equation*}

	Similar to Lemma 1 in \cite{huang2008functional}, $\hat{\tilde{\pmb{v}}}^{(-{d})}$ also solves  \eqref{pen ls eq} when the ${d}^{th}$ block of $\bar{\pmb{y}}$ is replaced by $\pmb{u}\hat{\tilde{{v}}}^{(-{d})}_{{d}}$. By partitioning the hat matrix $\pmb{H}$ into $D \times D$ equal-sized blocks, the ${d}^{th}$ block of $\hat{\bar{\pmb{y}}}$ can be expressed as:
    \begin{equation*}
        \hat{\tilde{{v}}}^{(-{d})}_{{d}}\pmb{u}= \sum_{\tilde d \neq {d}} \pmb{H}_{{d\tilde{d}}} \tilde{\pmb{C}}_{.{d}} + \pmb{H}_{{dd}}\pmb{u}\hat{\tilde{{v}}}^{(-{d})}_{{d}}.
    \end{equation*}
    Subtracting $\tilde{\pmb{C}}_{.{d}}$ from both sides and noting that $\sum_{\tilde d} \pmb{H}_{{d\tilde{d}}} \tilde{\pmb{C}}_{.{d}}$ = $\pmb{u}\hat{\tilde{{v}}}_{{d}}$, we obtain:
    \allowdisplaybreaks
    \begin{align*}  
   \hat{\tilde{{v}}}^{(-{d})}_{{d}}\pmb{u}-  \tilde{\pmb{C}}_{.{d}} 
   &= 
   \sum_{\tilde d} \pmb{H}_{{d\tilde{d}}} \tilde{\pmb{C}}_{.{d}} - \tilde{\pmb{C}}_{.{d}} + \pmb{H}_{{dd}}\hat{\tilde{{v}}}^{(-{d})}_{{d}}\pmb{u}- \pmb{H}_{{dd}} \tilde{\pmb{C}}_{.{d}}\\
    &=
    \sum_{\tilde d} \pmb{H}_{{d\tilde{d}}} \tilde{\pmb{C}}_{.{d}} - \tilde{\pmb{C}}_{.{d}} + \pmb{H}_{{dd}}(\hat{\tilde{{v}}}^{(-{d})}_{{d}}\pmb{u}-  \tilde{\pmb{C}}_{.{d}})\\
    &=
    \pmb{u}\hat{\tilde{{v}}}_{{d}} - \tilde{\pmb{C}}_{.{d}} + \pmb{H}_{{dd}}(\hat{\tilde{{v}}}^{(-{d})}_{{d}}\pmb{u}-  \tilde{\pmb{C}}_{.{d}}).
    \end{align*}
    Therefore, the cross-validation residual is given by:
\begin{equation*}
  \begin{split}
		     \hat{\tilde{{v}}}^{(-{d})}_{{d}}\pmb{u}-  \tilde{\pmb{C}}_{.{d}} = (\pmb{I}-\pmb{H}_{{dd}})^{-1}(\pmb{u}\hat{\tilde{{v}}}_{{d}} - \tilde{\pmb{C}}_{.{d}}), \\
    \end{split}
\end{equation*}
where $\pmb{H}_{{dd}}$ = ${{\lambda^{*}_d}\pmb{u}\pmb{u}^{\top}}.$
Using the Woodbury identity, we obtain:
\begin{align}\label{w1_single}
	(\pmb{I}-\pmb{H}_{{dd}})^{-1} = \pmb{I} +{\frac{{\lambda^{*}_d}}{1-{\lambda^{*}_d}\Vert\pmb{u}\Vert^2}}\pmb{u}\pmb{u}^{\top}.   
\end{align}
Denoting
	\begin{align*}
		\pmb{s} ={\pmb{u}\hat{\tilde{{v}}}_{{d}}}-\tilde{\pmb{C}}_{.{d}}, 
	\end{align*}
its squared norm is given by:
\begin{align}\label{w2_single}
\Vert\pmb{s}\Vert^2 
&= \tilde{\pmb{C}}_{.{d}}^{\top}\tilde{\pmb{C}}_{.{d}}-2{\tilde{\pmb{C}}_{.{d}}^{\top}\pmb{u}\hat{\tilde{{v}}}_{{d}}}+ {\hat{\tilde{{v}}}_{{d}}\pmb{u}^{\top}\pmb{u}\hat{\tilde{{v}}}_{{d}}}\notag\\
&=
\tilde{\pmb{C}}_{.{d}}^{\top}\tilde{\pmb{C}}_{.{d}}{-\frac{(\tilde{\pmb{C}}_{.{d}}^{\top}\pmb{u})^2}{\Vert\pmb{u}\Vert^2}+\left(\Vert\pmb{u}\Vert\hat{\tilde{{v}}}_{{d}}-\frac{\pmb{u}^{\top}\tilde{\pmb{C}}_{.{d}}}{\Vert\pmb{u}\Vert}\right)^2}.
\end{align}
Additionally, since
\begin{align}\label{w3_single}
\frac{(\pmb{u}^{\top}\pmb{s})^2}{\Vert\pmb{u}\Vert^2}
&=						\Vert\pmb{u}\Vert^2\left(\hat{\tilde{{v}}}_{{d}}-\frac{\pmb{u}^{\top}\tilde{\pmb{C}}_{.{d}}}{\Vert\pmb{u}\Vert^2}\right)^2\notag\\
&= 
\left(\Vert\pmb{u}\Vert\hat{\tilde{{v}}}_{{d}}-\frac{\pmb{u}^{\top}\tilde{\pmb{C}}_{.{d}}}{\Vert\pmb{u}\Vert}\right)^2,
\end{align}
we can rewrite equation (\ref{w2_single}) as:
\begin{align}\label{w4_single}
\Vert\pmb{s}\Vert^2 
&=
\tilde{\pmb{C}}_{.{d}}^{\top}\tilde{\pmb{C}}_{.{d}}{-\frac{(\tilde{\pmb{C}}_{.{d}}^{\top}\pmb{u})^2}{\Vert\pmb{u}\Vert^2}+\left(\Vert\pmb{u}\Vert\hat{\tilde{{v}}}_{{d}}-\frac{\pmb{u}^{\top}\tilde{\pmb{C}}_{.{d}}}{\Vert\pmb{u}\Vert}\right)^2}\notag\\
&=
\tilde{\pmb{C}}_{.{d}}^{\top}\tilde{\pmb{C}}_{.{d}}{-\frac{(\tilde{\pmb{C}}_{.{d}}^{\top}\pmb{u})^2}{\Vert\pmb{u}\Vert^2}+\frac{(\pmb{u}^{\top}\pmb{s})^2}{\Vert\pmb{u}\Vert^2}}.
\end{align}

Combining \eqref{w1_single}, \eqref{w3_single} and \eqref{w4_single}, we have:
\begin{align*}
\MoveEqLeft
\Vert\hat{\tilde{{v}}}^{(-{d})}_{{d}}\pmb{u}-  \tilde{\pmb{C}}_{.{d}}\Vert^2\\ 
&= 
\pmb{s}^{\top}\pmb{s}+{\frac{2{\lambda^{*}_d}}{1-{\lambda^{*}_d}\Vert\pmb{u}\Vert^2}}(\pmb{u}^{\top}\pmb{s})^2+{\frac{({\lambda^{*}_d})^2}{(1-{\lambda^{*}_d}\Vert\pmb{u}\Vert^2)^2}}(\pmb{u}^{\top}\pmb{s})^2\Vert\pmb{u}\Vert^2\notag\\
&= 
\pmb{s}^{\top}\pmb{s}+{\frac{(\pmb{u}^{\top}\pmb{s})^2}{\Vert\pmb{u}\Vert^2}\left(\frac{2{\lambda^{*}_d}}{1-{\lambda^{*}_d}\Vert\pmb{u}\Vert^2}\Vert\pmb{u}\Vert^2+\frac{({\lambda^{*}_d})^2}{(1-{\lambda^{*}_d}\Vert\pmb{u}\Vert^2)^2}\Vert\pmb{u}\Vert^4\right)}\notag\\
&=
\pmb{s}^{\top}\pmb{s}+{\frac{(\pmb{u}^{\top}\pmb{s})^2}{\Vert\pmb{u}\Vert^2}\frac{2{\lambda^{*}_d}\Vert\pmb{u}\Vert^2-({\lambda^{*}_d})^2\Vert\pmb{u}\Vert^4}{(1-{\lambda^{*}_d}\Vert\pmb{u}\Vert^2)^2}}\\
&=
\pmb{s}^{\top}\pmb{s}+{\frac{(\pmb{u}^{\top}\pmb{s})^2}{\Vert\pmb{u}\Vert^2}\left({\frac{1}{(1-{\lambda^{*}_d}\Vert\pmb{u}\Vert^2)^2}-1}\right)}\\
&=
\tilde{\pmb{C}}_{.{d}}^{\top}\tilde{\pmb{C}}_{.{d}}-{\dfrac{(\tilde{\pmb{C}}_{.{d}}^{\top}\pmb{u})^2}{\Vert\pmb{u}\Vert^2}}+{\frac{(\pmb{u}^{\top}\pmb{s})^2}{\Vert\pmb{u}\Vert^2}\frac{1}{(1-{\lambda^{*}_d}\Vert\pmb{u}\Vert^2)^2}}\\
&=
\tilde{\pmb{C}}_{.{d}}^{\top}\tilde{\pmb{C}}_{.{d}}-{\dfrac{(\tilde{\pmb{C}}_{.{d}}^{\top}\pmb{u})^2}{\Vert\pmb{u}\Vert^2}}+{\dfrac{\left(\Vert\pmb{u}\Vert\hat{\tilde{{v}}}_{{d}}-\dfrac{\tilde{\pmb{C}}_{.{d}}^{\top}\pmb{u}}{\Vert\pmb{u}\Vert}\right)^2}{(1-{\lambda^{*}_d}\Vert\pmb{u}\Vert^2)^2}}.
\end{align*}
Proof done.
\end{proof}

\begin{proof}[Proof of Theorem \ref{adjusted variance explained}]
Given that $\pmb{x}^{(r)}_i$ and $\pmb{x}^{(r+1)}_i$ represent the projections of $\pmb{x}_i$ onto $r$-dimensional subspace and $(r+1)$-dimensional subspace, respectively, the result follows directly. This observation confirms the validity of the theorem.
\end{proof}

\section{Derivation of CV and GCV Criteria in the Joint Power Algorithm} \label{joint GCV proof}

Denote the design matrix as \begin{equation*}
					\bar{\pmb{X}} := \begin{bmatrix}
						\bar{\pmb{X}}_1, 
						\dots,
						\bar{\pmb{X}}_{Q}
					\end{bmatrix} \in \mathbb{R}^{nD\times {Q} D},
				\end{equation*}

where $\bar{\pmb{X}}_{q} := \diag\{ \pmb{u}_{Q}, \ldots, \pmb{u}_{Q} \} \in \mathbb{R}^{nD\times D}$. With $\pmb{U}$ fixed, let ${\textsubtilde{\pmb{$V$}}} := \left({\tilde{\pmb{v}}_{1}}^\top, \ldots, {\tilde{\pmb{v}}_{{Q}}}^\top\right)^{\top} \in \mathbb{R}^{\ell d}$. Define $\bar{\pmb{\Omega}}_{\pmb{\alpha}} := (\pmb{U}^{\top}\pmb{U}) \otimes \pmb{\Omega_\alpha}$. Consequently, \eqref{fpca_joint_criterion} can be rewritten as: 
				\begin{equation}\label{sparse_fpca_joint_regression}
					\min_{\textsubtilde{\pmb{$V$}}} \  \left\{ \Vert\bar{\pmb{y}}-\bar{\pmb{X}}\textsubtilde{\pmb{$V$}}\Vert^2+\textsubtilde{\pmb{$V$}}^{\top}\bar{\pmb{\Omega}}_{\pmb{\alpha}}\textsubtilde{\pmb{$V$}}\right\},
				\end{equation}
				with the associated penalized covariance matrix given by:
				\begin{equation*}
					\bar{\pmb{X}}^{\top}\bar{\pmb{X}}+\bar{\pmb{\Omega}}_{\pmb{\alpha}} = (\pmb{U}^{\top}\pmb{U})\otimes(\pmb{I}+\pmb{\Omega_\alpha}) = (\pmb{U}^{\top}\pmb{U})\otimes\tilde{\pmb{S}}^{-1}_{\pmb{\alpha}},
				\end{equation*}
                and thus the hat matrix of the ridge regression is:
				\begin{equation*}
					\pmb{H} = \bar{\pmb{X}}(\bar{\pmb{X}}^{\top}\bar{\pmb{X}}+\bar{\pmb{\Omega}}_{\pmb{\alpha}})^{-1}\bar{\pmb{X}}^{\top}
					=\sum_{{q}=1}^{Q}{\frac{1}{\Vert\pmb{u}_{q}\Vert^2}\bar{\pmb{X}}_{q}\tilde{\pmb{S}}_{\pmb{\alpha}}}\bar{\pmb{X}}_{q}^{\top}.
				\end{equation*}

We define $\hat{\textsubtilde{\pmb{$V$}}}$ as the estimation of $\textsubtilde{\pmb{$V$}}$ that minimizes \eqref{sparse_fpca_joint_regression}. Let $\hat{\textsubtilde{\pmb{$V$}}}^{(-{d})}$ represent the estimate obtained when the ${d}^{th}$ block of $\bar{\pmb{y}}$ and the corresponding rows of $\bar{\pmb{X}}$ are omitted. Partition $\hat{\textsubtilde{\pmb{$V$}}}$ into ${Q}$ equal-length blocks, where the $i^{th}$ block is $\hat{\tilde{\pmb{v}}}{_{q}}$, the estimation of $\tilde{\pmb{v}}{_{q}}$. Similarly, the $i^{th}$ block of $\hat{\textsubtilde{\pmb{$V$}}}^{(-{d})}$ is denoted as $\hat{\tilde{\pmb{v}}}{_{q}}^{(-{d})}$. We now introduce the following lemma regarding the cross-validation prediction errors:
    				\begin{lemma}\label{CVlemma_joint}
					Let $\hat{\tilde{v}}_{{qd}}$ be the ${d}^{th}$ element of  $\hat{\tilde{\pmb{v}}}_{q} = \dfrac{\tilde{\pmb{S}}_{\pmb{\alpha}}\tilde{\pmb{C}}^{\top}\pmb{u}_{q}}{{\pmb{u}_{q}}^{\top}\pmb{u}_{q}}$ and $\hat{\tilde{v}}_{{qd}}^{(-{d})}$ be the ${d}^{th}$ element of $\hat{\tilde{\pmb{v}}}{_{q}}^{(-{d})}$. The ${d}^{th}$ leave-one-out cross-validation prediction error sum of squares is given by:
					\begin{equation}\label{00cv-A}		\Vert\sum_{q=1}^{{Q}}\hat{\tilde{v}}_{{qd}}^{(-{d})}\pmb{u}_q-\tilde{\pmb{C}}_{.{d}}\Vert^2 = \tilde{\pmb{C}}_{.{d}}^{\top}\tilde{\pmb{C}}_{.{d}}-\sum_{q=1}^{{Q}}{\dfrac{(\tilde{\pmb{C}}_{.{d}}^{\top}\pmb{u}_q)^2}{\Vert\pmb{u}_q\Vert^2}}+\sum_{q=1}^{{Q}}{\dfrac{\left(\Vert\pmb{u}_q\Vert\hat{\tilde{v}}_{{qd}}-\dfrac{\tilde{\pmb{C}}_{.{d}}^{\top}\pmb{u}_q}{\Vert\pmb{u}_q\Vert}\right)^2}{(1-{\lambda^{*}_d}^{(q)}\Vert\pmb{u}_q\Vert^2)^2}},
					\end{equation}
					where ${\lambda^{*}_d}^{(q)}$ = $\dfrac{{\{\tilde{\pmb{S}}_{\pmb{\alpha}}}\}_{{dd}}}{\Vert\pmb{u}_q\Vert^2}.$
				\end{lemma}
				
\begin{proof}[Proof of Lemma \ref{CVlemma_joint}]

Similar to Lemma \ref{CVlemma}, $\hat{\textsubtilde{\pmb{$V$}}}^{(-{d})}$ also solves  \eqref{sparse_fpca_joint_regression} when the ${d}^{th}$ block of $\bar{\pmb{y}}$ is replaced by $\sum_{q=1}^{{Q}}\hat{\tilde{v}}_{{qd}}^{(-{d})}\pmb{u}_q$. By partitioning the hat matrix $\pmb{H}$ into $D \times D$ equal-sized blocks, we obtain the ${d}^{th}$ block of $\hat{\bar{\pmb{y}}}$ as:
    \begin{equation*}
        \sum_{q=1}^{{Q}}\hat{\tilde{v}}_{{qd}}^{(-{d})}\pmb{u}_q = \sum_{\tilde d \neq {d}} \pmb{H}_{{d\tilde{d}}} \tilde{\pmb{C}}_{.{d}} + \pmb{H}_{{dd}}\sum_{q=1}^{{Q}}\hat{\tilde{v}}_{{qd}}^{(-{d})}\pmb{u}_q.
    \end{equation*}
    Subtracting $\tilde{\pmb{C}}_{.{d}}$ from both sides of the above equation and noting that $\sum_{\tilde d} \pmb{H}_{{d\tilde{d}}} \tilde{\pmb{C}}_{.{d}}$ = $\sum_{q=1}^{{Q}}\hat{\tilde{v}}_{{qd}}\pmb{u}_q$, we have:
    \begin{align*}  
   \sum_{q=1}^{{Q}}\hat{\tilde{v}}_{{qd}}^{(-{d})}\pmb{u}_q -  \tilde{\pmb{C}}_{.{d}} 
   &= 
   \sum_{\tilde d} \pmb{H}_{{d\tilde{d}}} \tilde{\pmb{C}}_{.{d}} - \tilde{\pmb{C}}_{.{d}} + \pmb{H}_{{dd}}\sum_{q=1}^{{Q}}\hat{\tilde{v}}_{{qd}}^{(-{d})}\pmb{u}_q - \pmb{H}_{{dd}} \tilde{\pmb{C}}_{.{d}}\\
    &=
    \sum_{\tilde d} \pmb{H}_{{d\tilde{d}}} \tilde{\pmb{C}}_{.{d}} - \tilde{\pmb{C}}_{.{d}} + \pmb{H}_{{dd}}(\sum_{q=1}^{{Q}}\hat{\tilde{v}}_{{qd}}^{(-{d})}\pmb{u}_q -  \tilde{\pmb{C}}_{.{d}})\\
    &=
    \sum_{q=1}^{{Q}}\hat{\tilde{v}}_{{qd}}\pmb{u}_q - \tilde{\pmb{C}}_{.{d}} + \pmb{H}_{{dd}}(\sum_{q=1}^{{Q}}\hat{\tilde{v}}_{{qd}}^{(-{d})}\pmb{u}_q -  \tilde{\pmb{C}}_{.{d}}).
    \end{align*}
    Therefore, the cross-validation residual is given by:
\begin{equation*}
  \begin{split}
		     \sum_{q=1}^{{Q}}\hat{\tilde{v}}_{{qd}}^{(-{d})}\pmb{u}_q -  \tilde{\pmb{C}}_{.{d}} = (\pmb{I}-\pmb{H}_{{dd}})^{-1}(\sum_{q=1}^{{Q}}\hat{\tilde{v}}_{{qd}}\pmb{u}_q - \tilde{\pmb{C}}_{.{d}}), \\
    \end{split}
\end{equation*}
where $\pmb{H}_{{dd}}$ = $\sum_{q=1}^{{Q}}{{\lambda^{*}_d}^{(q)}\pmb{u}_q\pmb{u}_q^{\top}}.$
Applying the Woodbury identity, we obtain:
\begin{align}\label{w1}
	(\pmb{I}-\pmb{H}_{{dd}})^{-1} = \pmb{I} +\sum_{q=1}^{{Q}}{\frac{{\lambda^{*}_d}^{(q)}}{1-{\lambda^{*}_d}^{(q)}\Vert\pmb{u}_q\Vert^2}}\pmb{u}_q\pmb{u}_q^{\top}.   
\end{align}
Denote:
	\begin{align*}
		\pmb{s} =\sum_{q=1}^{{Q}}{\hat{\tilde{v}}_{{qd}}\pmb{u}_q}-\tilde{\pmb{C}}_{.{d}}, 
	\end{align*}
and its squared norm is:
\begin{align}\label{w2}
\Vert\pmb{s}\Vert^2 
&= \tilde{\pmb{C}}_{.{d}}^{\top}\tilde{\pmb{C}}_{.{d}}-2\sum_{q=1}^{{Q}}{\tilde{\pmb{C}}_{.{d}}^{\top}\hat{\tilde{v}}_{{qd}}\pmb{u}_q}+ \sum_{q=1}^{{Q}}{\hat{\tilde{v}}_{{qd}}\pmb{u}_q^{\top}\pmb{u}_q\hat{\tilde{v}}_{{qd}}}\notag\\
&=
\tilde{\pmb{C}}_{.{d}}^{\top}\tilde{\pmb{C}}_{.{d}}+\sum_{q=1}^{{Q}}{-\frac{(\tilde{\pmb{C}}_{.{d}}^{\top}\pmb{u}_q)^2}{\Vert\pmb{u}_q\Vert^2}+\left(\Vert\pmb{u}_q\Vert\hat{\tilde{v}}_{{qd}}-\frac{\pmb{u}_q^{\top}\tilde{\pmb{C}}_{.{d}}}{\Vert\pmb{u}_q\Vert}\right)^2}.
\end{align}
Additionally, since
\begin{align}\label{w3}
\frac{(\pmb{u}_q^{\top}\pmb{s})^2}{\Vert\pmb{u}_q\Vert^2}
&=						\Vert\pmb{u}_q\Vert^2\left(\hat{\tilde{v}}_{{qd}}-\frac{\pmb{u}_q^{\top}\tilde{\pmb{C}}_{.{d}}}{\Vert\pmb{u}_q\Vert^2}\right)^2\notag\\
&= 
\left(\Vert\pmb{u}_q\Vert\hat{\tilde{v}}_{{qd}}-\frac{\pmb{u}_q^{\top}\tilde{\pmb{C}}_{.{d}}}{\Vert\pmb{u}_q\Vert}\right)^2,
\end{align}
we can rewrite equation (\ref{w2}) as:
\begin{align}\label{w4}
\Vert\pmb{s}\Vert^2 
&=
\tilde{\pmb{C}}_{.{d}}^{\top}\tilde{\pmb{C}}_{.{d}}+\sum_{q=1}^Q{-\frac{(\tilde{\pmb{C}}_{.{d}}^{\top}\pmb{u}_q)^2}{\Vert\pmb{u}_q\Vert^2}+\left(\Vert\pmb{u}_q\Vert\hat{\tilde{v}}_{{qd}}-\frac{\pmb{u}_q^{\top}\tilde{\pmb{C}}_{.{d}}}{\Vert\pmb{u}_q\Vert}\right)^2}\notag\\
&=
\tilde{\pmb{C}}_{.{d}}^{\top}\tilde{\pmb{C}}_{.{d}}+\sum_{q=1}^Q{-\frac{(\tilde{\pmb{C}}_{.{d}}^{\top}\pmb{u}_q)^2}{\Vert\pmb{u}_q\Vert^2}+\frac{(\pmb{u}_q^{\top}\pmb{s})^2}{\Vert\pmb{u}_q\Vert^2}}.
\end{align}

Combining \eqref{w1}, \eqref{w3} and \eqref{w4}, we can obtain
\begin{align*}
\MoveEqLeft
\Vert\sum_{q=1}^{{Q}}\hat{\tilde{v}}_{{qd}}^{(-{d})}\pmb{u}_q -  \tilde{\pmb{C}}_{.{d}}\Vert^2\\ 
&= 
\pmb{s}^{\top}\pmb{s}+\sum_{q=1}^{{Q}}{\frac{2{\lambda^{*}_d}^{(q)}}{1-{\lambda^{*}_d}^{(q)}\Vert\pmb{u}_q\Vert^2}}(\pmb{u}_q^{\top}\pmb{s})^2+\sum_{q=1}^{{Q}}{\frac{({\lambda^{*}_d}^{(q)})^2}{(1-{\lambda^{*}_d}^{(q)}\Vert\pmb{u}_q\Vert^2)^2}}(\pmb{u}_q^{\top}\pmb{s})^2\Vert\pmb{u}_q\Vert^2\notag\\
&= 
\pmb{s}^{\top}\pmb{s}+\sum_{q=1}^{{Q}}{\frac{(\pmb{u}_q^{\top}\pmb{s})^2}{\Vert\pmb{u}_q\Vert^2}\left(\frac{2{\lambda^{*}_d}^{(q)}}{1-{\lambda^{*}_d}^{(q)}\Vert\pmb{u}_q\Vert^2}\Vert\pmb{u}_q\Vert^2+\frac{({\lambda^{*}_d}^{(q)})^2}{(1-{\lambda^{*}_d}^{(q)}\Vert\pmb{u}_q\Vert^2)^2}\Vert\pmb{u}_q\Vert^4\right)}\notag\\
&=
\pmb{s}^{\top}\pmb{s}+\sum_{q=1}^{{Q}}{\frac{(\pmb{u}_q^{\top}\pmb{s})^2}{\Vert\pmb{u}_q\Vert^2}\frac{2{\lambda^{*}_d}^{(q)}\Vert\pmb{u}_q\Vert^2-({\lambda^{*}_d}^{(q)})^2\Vert\pmb{u}_q\Vert^4}{(1-{\lambda^{*}_d}^{(q)}\Vert\pmb{u}_q\Vert^2)^2}}\\
&=
\pmb{s}^{\top}\pmb{s}+\sum_{q=1}^{{Q}}{\frac{(\pmb{u}_q^{\top}\pmb{s})^2}{\Vert\pmb{u}_q\Vert^2}\left(\frac{1}{(1-{\lambda^{*}_d}^{(q)}\Vert\pmb{u}_q\Vert^2)^2}-1\right)}\\
&=
\tilde{\pmb{C}}_{.{d}}^{\top}\tilde{\pmb{C}}_{.{d}}-\sum_{q=1}^{{Q}}{\dfrac{(\tilde{\pmb{C}}_{.{d}}^{\top}\pmb{u}_q)^2}{\Vert\pmb{u}_q\Vert^2}}+\sum_{q=1}^{{Q}}{\frac{(\pmb{u}_q^{\top}\pmb{s})^2}{\Vert\pmb{u}_q\Vert^2}\frac{1}{(1-{\lambda^{*}_d}^{(q)}\Vert\pmb{u}_q\Vert^2)^2}}\\
&=
\tilde{\pmb{C}}_{.{d}}^{\top}\tilde{\pmb{C}}_{.{d}}-\sum_{q=1}^{{Q}}{\dfrac{(\tilde{\pmb{C}}_{.{d}}^{\top}\pmb{u}_q)^2}{\Vert\pmb{u}_q\Vert^2}}+\sum_{q=1}^{{Q}}{\dfrac{\left(\Vert\pmb{u}_q\Vert\hat{\tilde{v}}_{{qd}}-\dfrac{\tilde{\pmb{C}}_{.{d}}^{\top}\pmb{u}_q}{\Vert\pmb{u}_q\Vert}\right)^2}{(1-{\lambda^{*}_d}^{(q)}\Vert\pmb{u}_q\Vert^2)^2}}.
\end{align*}
Proof done.
\end{proof}

Given that $\tilde{\pmb{C}}$ and $\pmb{U}$ are fixed, the expressions \eqref{CV score2} and \eqref{GCV score2} are obtained by averaging the final terms on the right-hand side of \eqref{00cv-A}.

\section{Half Smoothing Implementation}
Here we present the implementation details for Corollary \ref{Corollary SVD ED}. For a given smoothing parameter $\pmb{\alpha}$, the optimal solutions for $\pmb{u}$ and $\pmb{\tilde{\varphi}}$ as stated in Corollary \ref{Corollary SVD ED} are derived from the leading pair of functional SVD components of $\tilde{\mathbfcal{X}}$. Upon obtaining $\tilde{\pmb{\varphi}}$, we apply a half-smoothing procedure to yield the smoothed functional PC, represented as $\pmb{\varphi} = \mathbfcal{S}_{\pmb{\alpha}}\tilde{\pmb{\varphi}}$. 

In the finite-dimensional functional space framework, $\tilde{\varphi}_j$ can be expressed as $\tilde{\varphi}_j = \pmb{\phi}_j(t)^\top\bar{\pmb{v}}_{j}$, where $\bar{\pmb{v}}_{j} \in \mathbb{R}^{d_j}$. Similarly, each ${{\tilde{x}}}_{i,j}$ can be represented by ${{\tilde{x}}}_{i,j} = \pmb{\nu}_{j}(t)^\top\bar{\pmb{c}}_{i,j}$, with $\bar{\pmb{c}}_{i,j} \in \mathbb{R}^{d_j}$. Under this framework, the inner product $\langle\tilde{\pmb{\varphi}},\tilde{\pmb{\varphi}}\rangle_{H}$ can be represented as $\bar{\pmb{v}}^{\top}\pmb{G}\bar{\pmb{v}}$, where $\bar{\pmb{v}} = (\bar{\pmb{v}}^{\top}_1, \ldots, \bar{\pmb{v}}^{\top}_p)^{\top}$. Additionally, this is equivalent to $\langle\tilde{\pmb{\varphi}},\tilde{\pmb{\varphi}}\rangle_{H} = {{\pmb{v}}^{\top}\pmb{S}_{\pmb{\alpha}}^{-2}\pmb{v}}$. By equating ${\bar{\pmb{v}}^{\top}\pmb{G}\bar{\pmb{v}}}$ with ${\pmb{v}}^{\top}{\pmb{S}_{\pmb{\alpha}}^{-2}}{\pmb{v}}$, we derive the relationship $\bar{\pmb{v}} = \pmb{G}^{-\frac{1}{2}}{\pmb{S}_{\pmb{\alpha}}^{-1}}\pmb{v}$. Similarly, we can establish that $\bar{\pmb{c}}{_{j}} = \pmb{G}^{-\frac{1}{2}}{\pmb{S}_{\pmb{\alpha}}}\pmb{G}\pmb{c}_{j}$, where $\bar{\pmb{c}}_i = (\bar{\pmb{c}}_{i,1}^\top, \ldots, \bar{\pmb{c}}_{i,p}^\top)^\top$ and $\bar{\pmb{C}} = \pmb{C}\pmb{G}\pmb{S}_{\pmb{\alpha}}\pmb{G}^{-\frac{1}{2}}$.

Using Theorem 3 from \cite{Haghbin2019fssa}, we initially obtain $\bar{\pmb{v}} = \pmb{G}^{-\frac{1}{2}}\bar{\pmb{u}}$, where $\bar{\pmb{u}}$ denotes the first right singular vectors of $\pmb{X} = \pmb{G}^{\frac{1}{2}}\bar{\pmb{C}}^{\top}$. Subsequently, we compute $\pmb{v} = {\pmb{S}_{\pmb{\alpha}}}\pmb{G}^{\frac{1}{2}}\bar{\pmb{v}}$.

Furthermore, multiple functional PCs and their corresponding scores can be simultaneously derived as the leading pairs of functional SVD components. When estimating the first $Q$ functional SVD components concurrently, we calculate $\pmb{V} = {\pmb{S}_{\pmb{\alpha}}}\pmb{G}^{\frac{1}{2}}\bar{\pmb{V}}$, where $\bar{\pmb{V}} = [\bar{\pmb{v}}_1,\dots,\bar{\pmb{v}}_Q]$.

\newpage

\section{Supplementary Figures}
\begin{figure}[!h] 
\centering
  \includegraphics[width=1\linewidth]{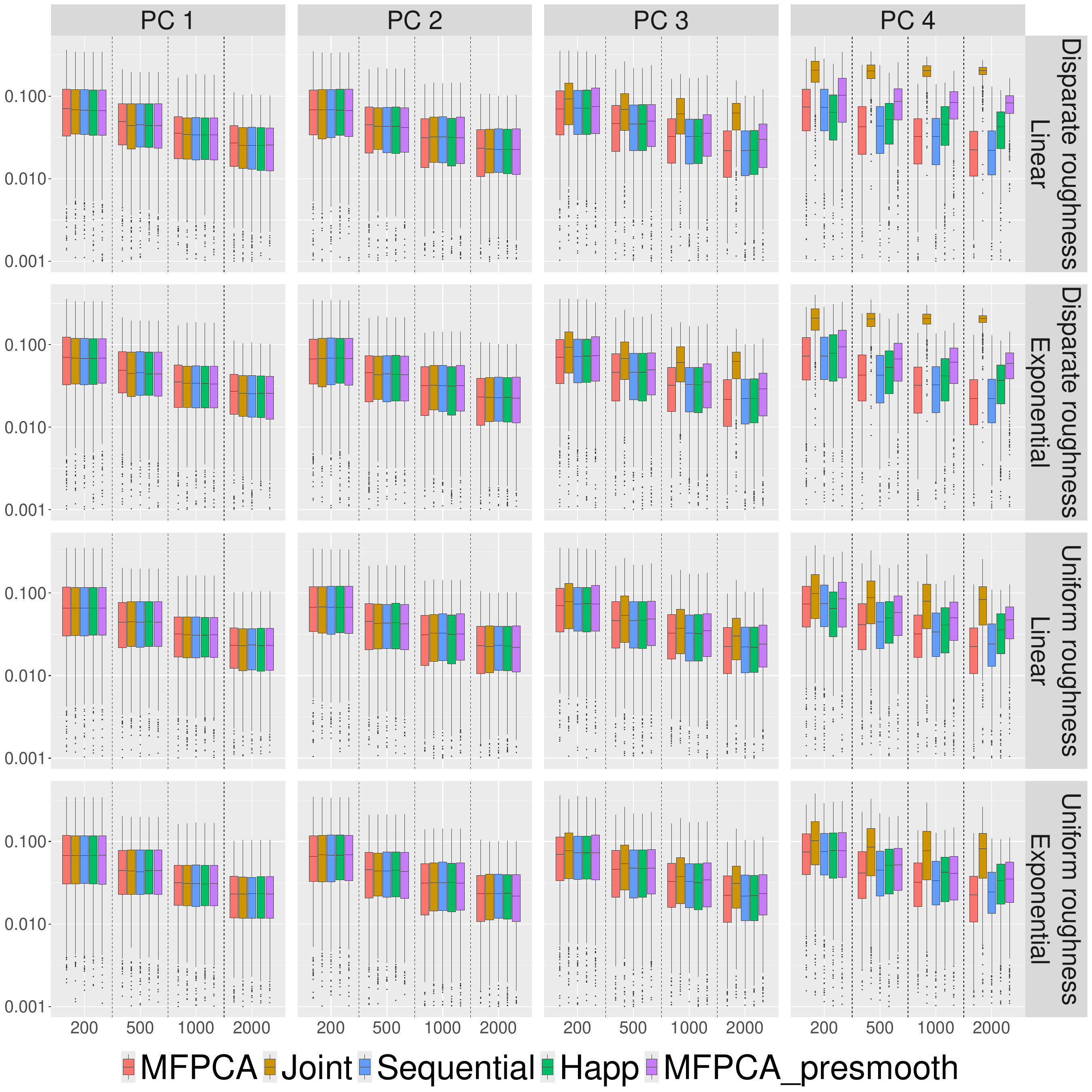}
  \caption{Relative eigenvalue estimation errors across the simulation settings. 
The error is measured by 
$Err(\hat{\lambda}_m)=|\hat{\lambda}_m-\lambda_m|/|\lambda_m|$ 
for the leading functional principal components and is compared across the competing MFPCA and ReMFPCA methods.}
  \label{eigenvalue error}
\end{figure}

\begin{figure}[!t] 
  \centering
  \begin{subfigure}{0.5\textwidth}
    \centering
    \includegraphics[width=\textwidth]{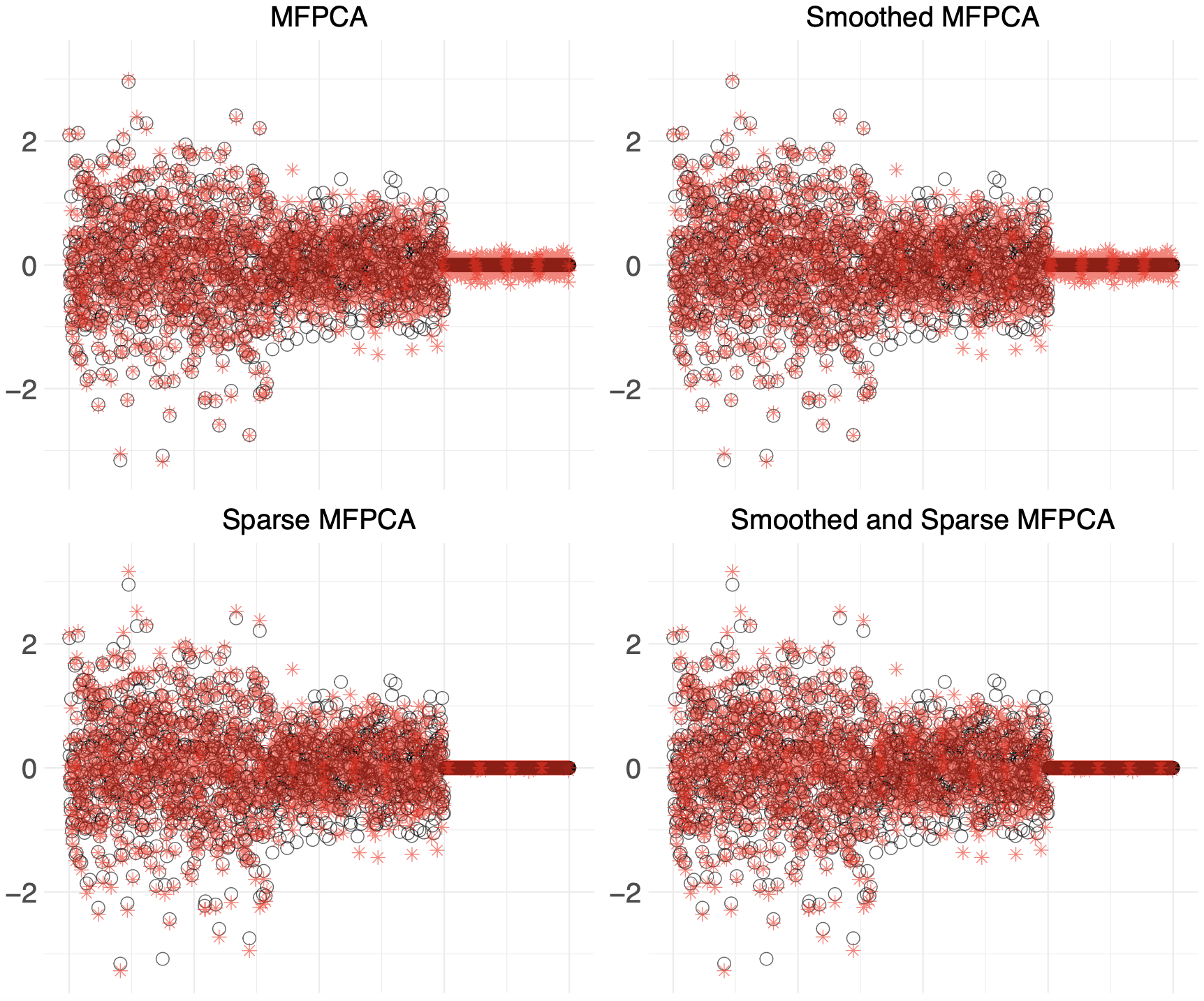}
    \caption{PC 1 scores}
    \label{pc1 score}
  \end{subfigure}%
  \begin{subfigure}{0.5\textwidth}
    \centering
    \includegraphics[width=\textwidth]{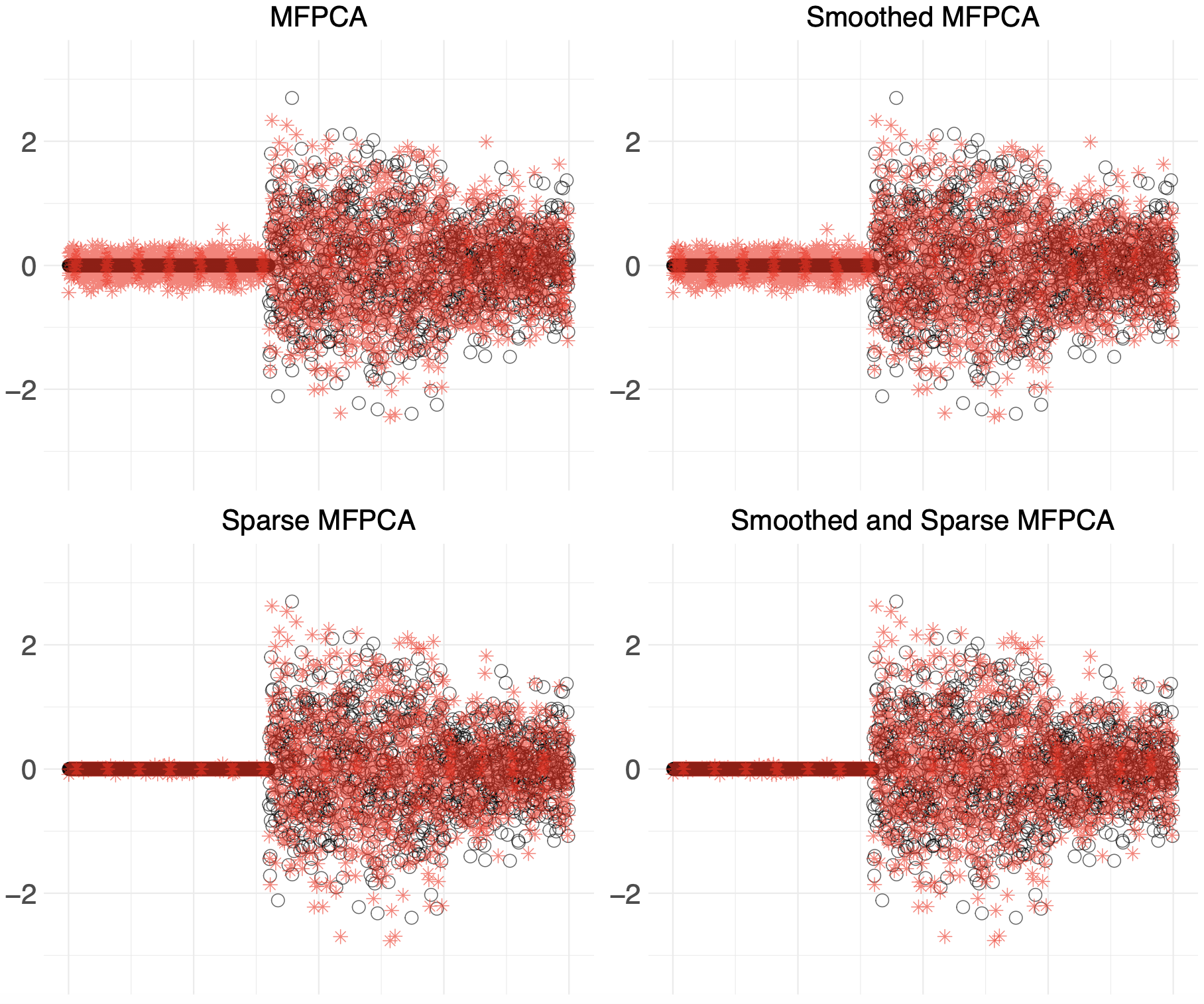}
    \caption{PC 2 scores}
    \label{pc2 score}
  \end{subfigure}
  \begin{subfigure}{0.5\textwidth}
    \centering
    \includegraphics[width=\textwidth]{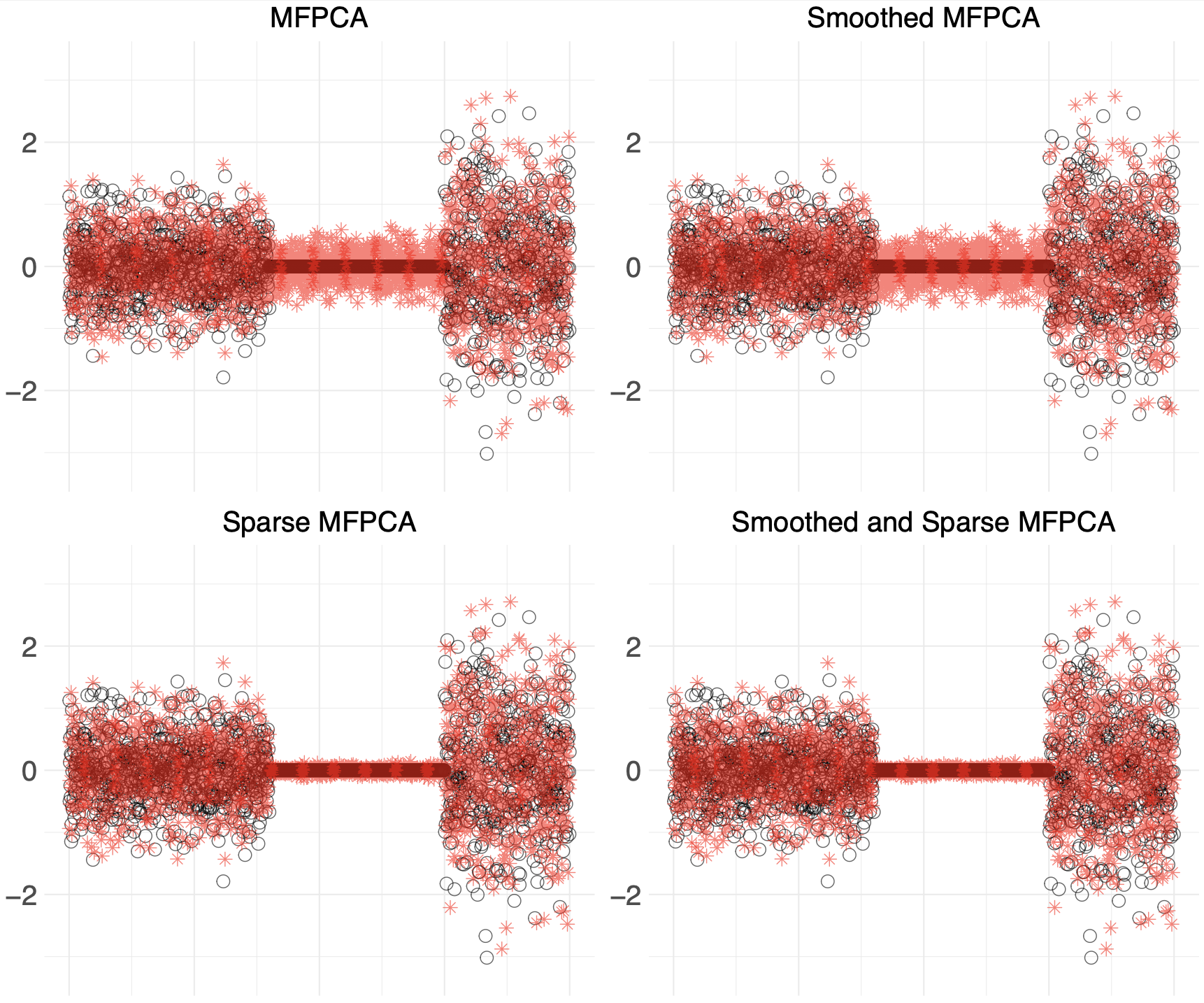}
    \caption{PC 3 scores}
    \label{pc3 score}
  \end{subfigure}
    \begin{subfigure}{0.49\textwidth}
    \centering
    \includegraphics[width=\textwidth]{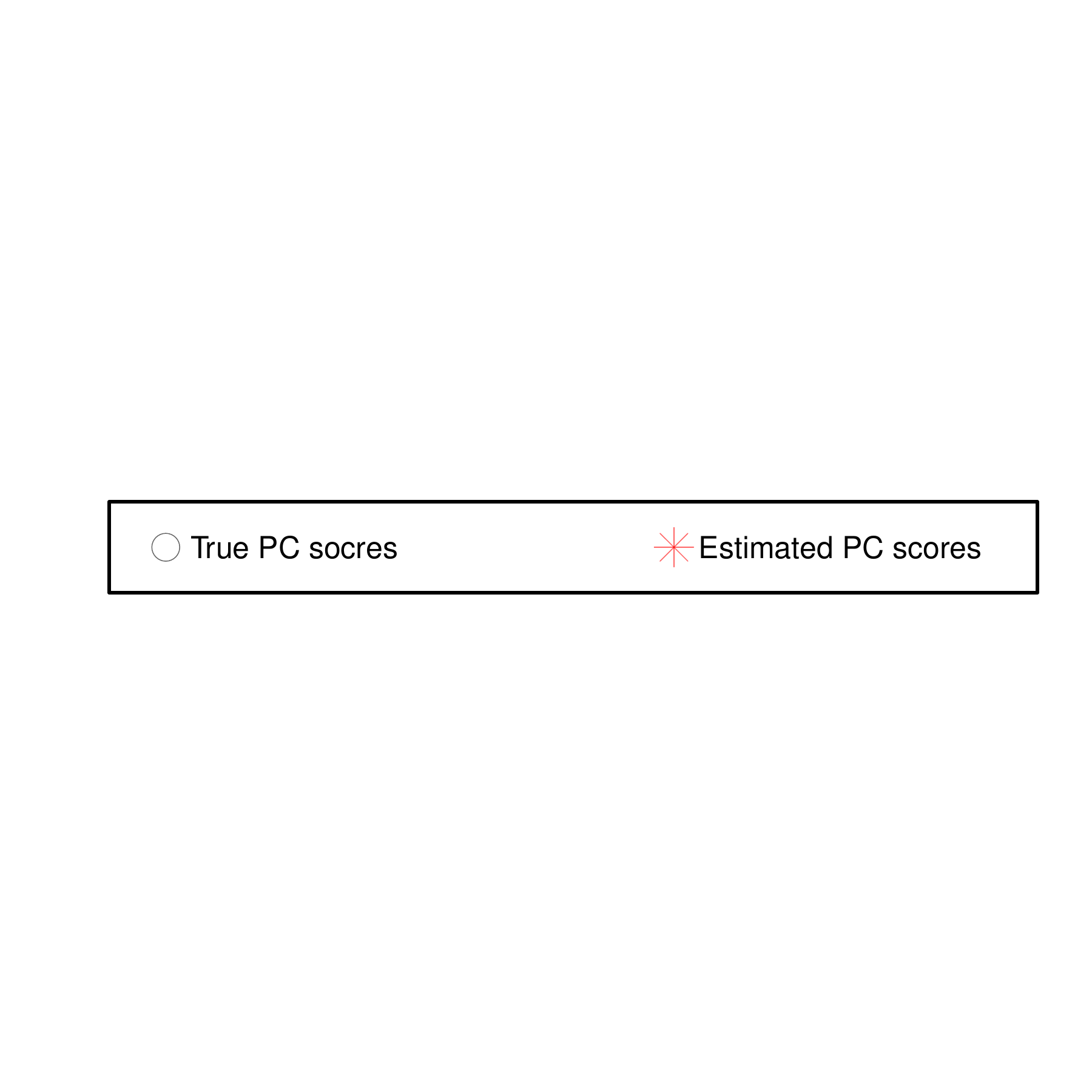}
  \end{subfigure}
  \caption{Estimated and true PC scores from a randomly selected simulation replicate under the high-sparsity setting, Scenario III. 
The estimated PC scores are shown in red, while the true PC scores are shown in black. 
The panels display the first three PC scores and illustrate the ability of the proposed smooth-and-sparse ReMFPCA method to recover the underlying sparse score structure.}
  \label{pc score}
\end{figure}

\begin{figure}[!t]
    \centering
    \includegraphics[width=1\textwidth]{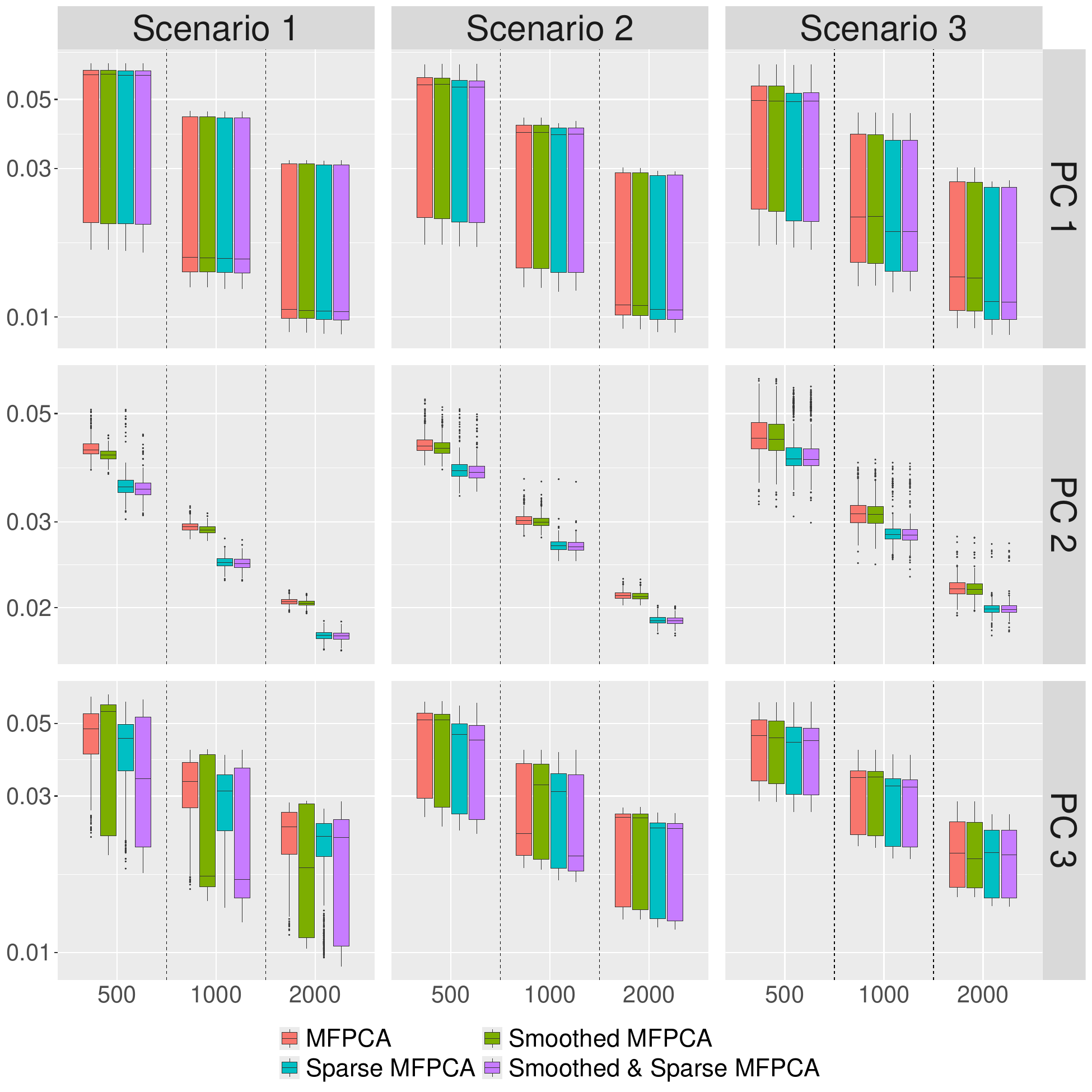}
    \caption{Estimation error of the sparse PC scores across the simulation settings. 
The error is measured by 
$Err(\hat{\rho}_{i,m}^{k}) = N^{-1}\sum_{i=1}^{N}|\hat{\rho}_{i,m}^{k}-\rho_{i,m}^{k}|$, 
summarizing the discrepancy between the estimated and true subject-specific PC scores for the leading components. 
The results compare the accuracy of the competing methods under different sparsity levels and sample sizes.}
    \label{sparse score error}
\end{figure}

\begin{figure}[!t]
    \centering
    \includegraphics[width=1\textwidth]{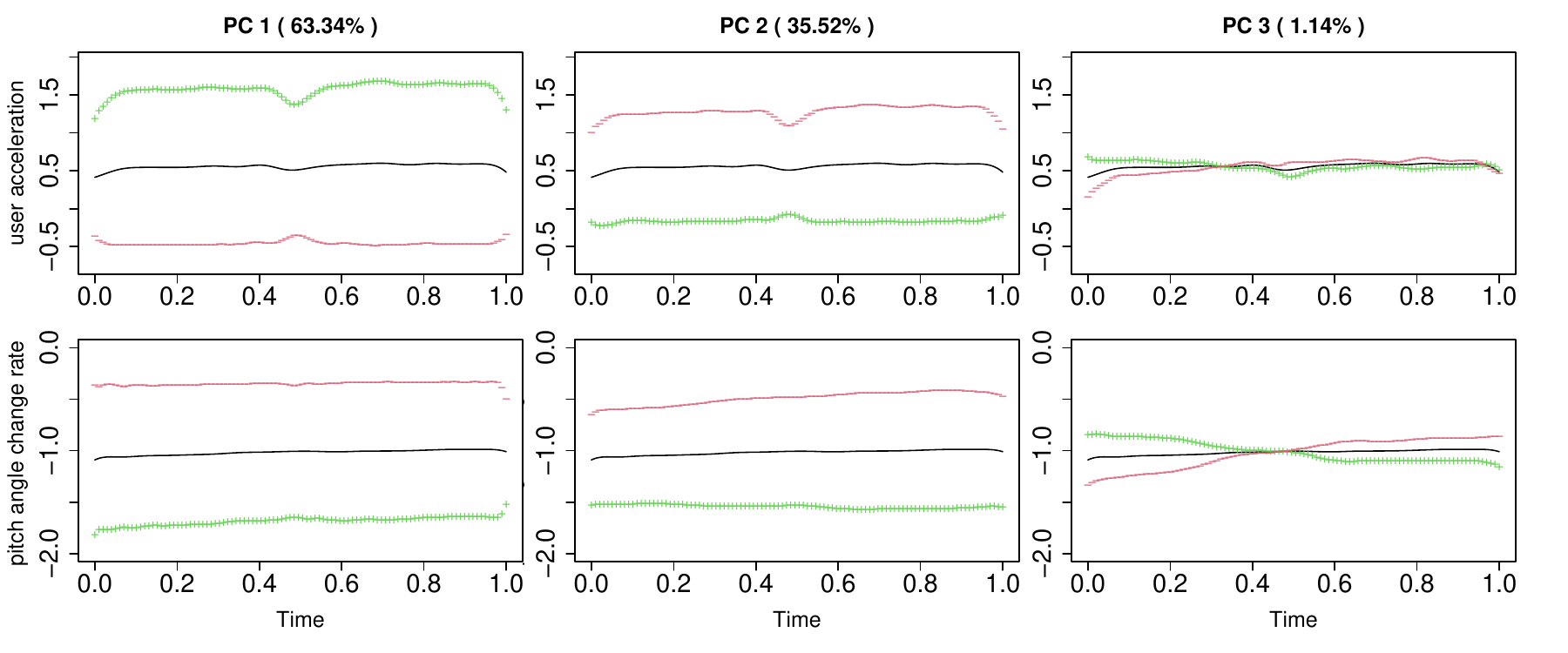}
    \caption{Effect of adding and subtracting a suitable multiple of each of the first three estimated functional PCs to the overall mean curve. 
The figure illustrates the dominant modes of variation captured by the proposed smooth-and-sparse ReMFPCA method in the MotionSense data.}
    \label{FPC plot:plot5}
\end{figure}

\begin{figure}[!h] 
\centering
  \includegraphics[width=1\linewidth]{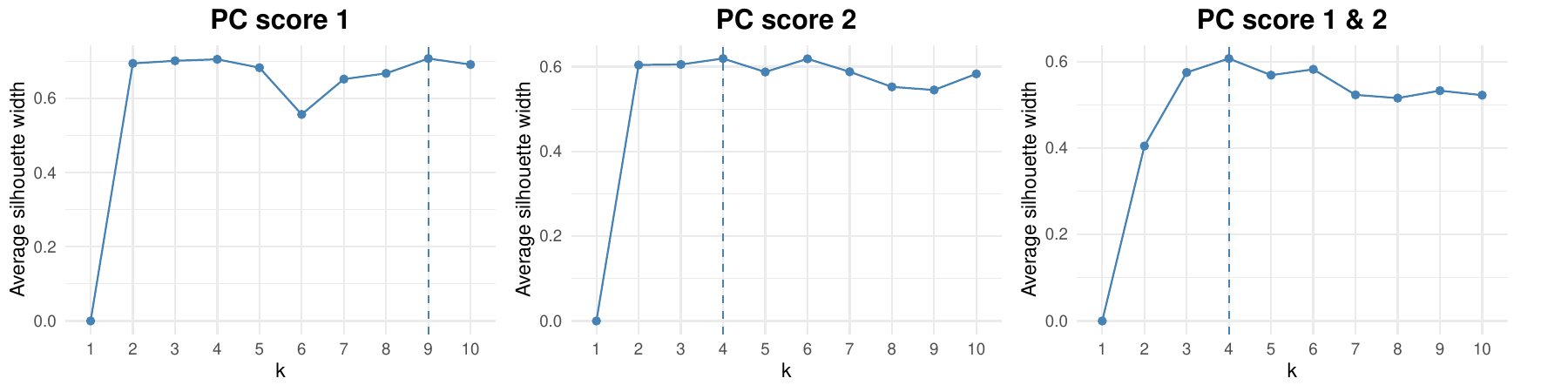}
  \caption{Silhouette analysis for selecting the number of clusters based on the estimated ReMFPCA scores from the MotionSense data. 
The criterion is evaluated using the first PC score, the second PC score, and the first two PC scores jointly.}
  \label{Silhouette plot}
\end{figure}

\begin{figure}[!t]
    \centering
    \includegraphics[width=1\textwidth]{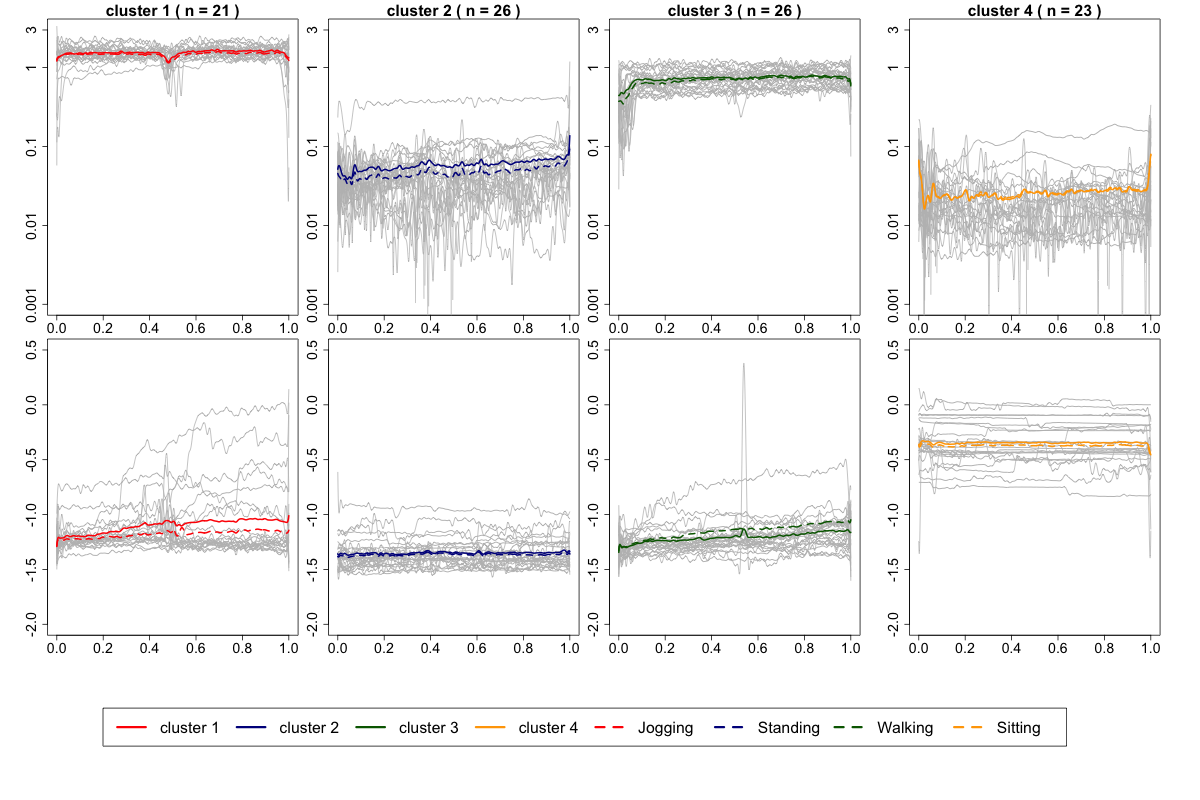}
    \caption{Clustering results for the MotionSense data based on the first two ReMFPCA scores. 
For each estimated cluster, the solid curve represents the estimated cluster mean, while the dashed curve represents the corresponding true activity-specific mean curve.}
    \label{cluster}
\end{figure}